\documentclass[10pt,journal,compsoc]{IEEEtran}
\usepackage{cite}
\usepackage{amsmath,amssymb,amsfonts}
\usepackage{algorithmic}
\usepackage{enumitem}
\usepackage{graphicx}
\usepackage{textcomp}
\usepackage{hhline}
\usepackage{lipsum} 
\usepackage{amsmath, amsthm, amssymb, amsfonts}
\usepackage{bm}
\usepackage[makeroom]{cancel}
\newcommand\ci{\perp\!\!\!\perp}
\DeclareMathOperator*{\argmin}{arg\,min}

\usepackage[dvipsnames]{xcolor}

\usepackage{tikz}
\usetikzlibrary{arrows, automata}
\usetikzlibrary{decorations.pathreplacing,angles,quotes,arrows.meta,%
                backgrounds,
                calligraphy,
                positioning,calc}

\makeatletter
\newtheorem*{rep@theorem}{\rep@title}
\newcommand{\newreptheorem}[2]{%
\newenvironment{rep#1}[1]{%
 \def\rep@title{#2 \ref{##1}}%
 \begin{rep@theorem}}%
 {\end{rep@theorem}}}
\makeatother

\newreptheorem{theorem}{Theorem}
\newreptheorem{lemma}{Lemma}
\newreptheorem{proposition}{Proposition}

\newtheorem{theorem}{Theorem}
\newtheorem{lemma}{Lemma}

\newtheorem{definition}{Definition}
\usepackage{xcolor}
\usepackage{blindtext}
\usepackage[most]{tcolorbox}
\usepackage[ruled,linesnumbered]{algorithm2e}
\usepackage{soul}

	\definecolor{burntorange}{rgb}{0.8, 0.33, 0.0}

\usepackage{subcaption}
\usepackage{xcolor,colortbl}

\let\oldnl\nl
\newcommand{\nonl}{\renewcommand{\nl}{\let\nl\oldnl}}

\pgfdeclaredecoration{lightning bolt}{draw}{
\state{draw}[width=\pgfdecoratedpathlength]{
  \pgfpathmoveto{\pgfpointorigin}%
  \pgfpathlineto{\pgfpoint{\pgfdecoratedpathlength*0.6}%
    {-\pgfdecoratedpathlength*.1}}%
  \pgfpathlineto{\pgfpoint{\pgfdecoratedpathlength*0.55}{0pt}}%
  \pgfpathlineto{\pgfpoint{\pgfdecoratedpathlength}{0pt}}%
  \pgfpathlineto{\pgfpoint{\pgfdecoratedpathlength*0.4}%
    {\pgfdecoratedpathlength*.1}}%
  \pgfpathlineto{\pgfpoint{\pgfdecoratedpathlength*0.45}{0pt}}%
  \pgfpathclose%
}%
}

\begin{document}
\title{Identifying Patient-Specific Root Causes\\with the Heteroscedastic Noise Model}
\author{Eric V. Strobl, Thomas A. Lasko}

\maketitle

\begin{abstract}
Complex diseases are caused by a multitude of factors that may differ between patients even within the same diagnostic category. A few underlying \textit{root causes} may nevertheless initiate the development of disease within each patient. We therefore focus on identifying \textit{patient-specific} root causes of disease, which we equate to the sample-specific predictivity of the exogenous error terms in a structural equation model. We generalize from the linear setting to the \textit{heteroscedastic noise model} where $Y = m(X) + \varepsilon\sigma(X)$ with non-linear functions $m(X)$ and $\sigma(X)$ representing the conditional mean and mean absolute deviation, respectively. This model preserves identifiability but introduces non-trivial challenges that require a customized algorithm called Generalized Root Causal Inference (GRCI) to extract the error terms correctly. GRCI recovers patient-specific root causes more accurately than existing alternatives. 
\end{abstract}


\section{Introduction}
\label{sec:introduction}
\IEEEPARstart{C}{ausal} inference refers to the process of inferring causal relationships from data. Randomized controlled trials (RCTs) remain the gold standard for causal inference in most fields of science. However, RCTs cannot distinguish between causes and \textit{root causes} of disease, or the initial perturbations to a biological system that ultimately induce a diagnostic label as a downstream effect; we will clarify this definition in Section \ref{sec_root}. Randomization also introduces a myriad of ethical, financial and logistical issues -- such as withholding potentially lifesaving treatments from patients. We therefore instead focus on identifying root causes from \textit{observational data}, where patients are not subject to randomization.

Consider for example the causal process depicted by the directed graph in Figure \ref{fig_root_cause}, where nodes represent random variables and directed edges their direct causal relations. The blue lightning bolt depicts an exogenous ``shock'' to the causal process, such as the effect of a somatic mutation or a virus on the expression level of a gene $X_2 \in \bm{X}$. The shock is felt by downstream genes $X_3,X_4$ ultimately generating symptoms $X_5, X_6$ and then causing a clinician to label a patient with a diagnosis $D$ based on the symptoms. We focus on identifying $X_2$ from data because it corresponds to the initial perturbation and therefore the root cause. The problem is challenging because the root cause may lie arbitrarily far from $D$, and we must differentiate it from the other variables in $\bm{X}$ that may be causes but not necessarily the root cause of the diagnosis.

The problem is further complicated by the existence of complex diseases that may have multiple root causes differing between patients even within the same diagnostic category. As a result, simply identifying the root causes of all patients with the same diagnosis can lead to many statistically significant variables with clinically insignificant effect sizes. We instead focus our efforts on identifying \textit{patient-specific} root causes in order to make complex diseases more tractable.

We identify patient-specific root causes by first defining a causal process using a structural equation model (SEM), where variables are related by a series of deterministic equations and stochastic error terms. Patient-specific root causes then correspond to the predictivity of the exogenous errors as assessed by Shapley values (see Section \ref{sec_root} for details). Obtaining these exogenous errors requires invertible SEMs, so that we can recover the error term values uniquely from the observed variables. For example, the linear non-Gaussian acyclic model (LiNGAM) is an invertible SEM with linear equations and non-Gaussian error terms \cite{Shimizu06}. Authors have thus far only utilized LiNGAM to recover the error terms and infer patient-specific root causes of disease \cite{Strobl22}. 

Real datasets however frequently contain non-linear relations, and running linear algorithms on data sampled from a non-linear SEM can lead to large errors in estimation. Investigators have introduced several approaches towards handling non-linear relations. The additive noise model (ANM) considers $Y=m(X) + \varepsilon$, where $m(X)$ denotes a possibly non-linear function and the conditional variance of $\varepsilon$ does not depend on $X$ \cite{Hoyer08}. The post-nonlinear model (PNL) extends ANM by introducing an outer invertible transformation $h$ so that $Y=h(m(X) + \varepsilon)$ and the conditional variance of $\varepsilon$ may either monotonically increase or decrease as a function of $X$ \cite{Zhang09}. Existing methods therefore impose strong restrictions on the conditional variance of the error term.

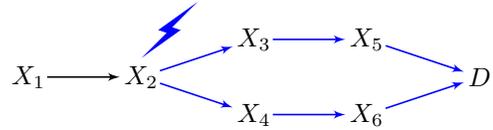
\begin{figure}
\centering
\begin{tikzpicture}[scale=1.0, shorten >=1pt,auto,node distance=2.8cm, semithick,
  inj/.pic = {\draw (0,0) -- ++ (0,2mm) 
                node[minimum size=2mm, fill=red!60,above] {}
                node[draw, semithick, minimum width=2mm, minimum height=5mm,above] (aux) {};
              \draw[thick] (aux.west) -- (aux.east); 
              \draw[thick,{Bar[width=2mm]}-{Hooks[width=4mm]}] (aux.center) -- ++ (0,4mm) coordinate (-inj);
              }]
                    
\tikzset{vertex/.style = {inner sep=0.4pt}}
\tikzset{edge/.style = {->,> = latex'}}
 
\node[vertex] (1) at  (0,0) {$X_1$};
\node[vertex] (2) at  (1.5,0) {$X_2$};
\node[vertex] (3) at  (3,0.5) {$X_3$};
\node[vertex] (4) at  (3,-0.5) {$X_4$};
\node[vertex] (5) at  (4.5,0.5) {$X_5$};
\node[vertex] (6) at  (4.5,-0.5) {$X_6$};
\node[vertex] (7) at  (6,0) {$D$};

\fill [blue, decoration=lightning bolt, decorate] (1.5,0.25) -- ++ (0.75,0.75);

\draw[edge] (1) to (2);
\draw[edge,blue] (2) to (3);
\draw[edge,blue] (2) to (4);
\draw[edge,blue] (3) to (5);
\draw[edge,blue] (4) to (6);
\draw[edge,blue] (5) to (7);
\draw[edge,blue] (6) to (7);
\end{tikzpicture}
\caption{Intuitive illustration of the difference between a cause and patient-specific root cause.} \label{fig_root_cause}
\end{figure}

\begin{tcolorbox}[breakable,enhanced,frame hidden]
In this paper, we infer patient-specific root causes while allowing the variance of the error term to change arbitrarily as a function of $X$ via multiple innovations:
\begin{enumerate}[leftmargin=*,label=(\arabic*)]
    \item We consider the \textit{heteroscedastic noise model} (HNM) given by $Y = m(X) + \varepsilon \sigma(X)$ with arbitrary non-linear functions $m$ and $\sigma$ representing the conditional mean and conditional mean absolute deviation (MAD), respectively  (Section \ref{sec_HNM}).
    \item We prove identifiability of the full causal graph under HNM (Section \ref{sec_identify}) and introduce a cross-validation procedure to efficiently extract the error terms of HNM (Section \ref{sec_HNM:extract}).
    \item We quantify root causal contributions using Shapley values based on conditional distributions on the error terms; these values accommodate noisy labels, fast computation and differing prevalence rates without requiring additional background knowledge (Section \ref{sec_root:def}).
    \item We introduce an algorithm called Generalized Root Causal Inference (GRCI) that efficiently extracts the error terms of an SEM satisfying HNM using spline functions and quickly computes the proposed Shapley values all without access to the underlying causal graph (Section \ref{sec_GRCI}).
    \end{enumerate}

\end{tcolorbox}
\noindent Experiments highlight considerable improvements in accuracy compared to prior methods because GRCI correctly identifies the exogenous errors by flexibly accounting for nonlinear causal relations. We emphasize that, while this paper focuses on automatically identifying patient-specific root causes of disease -- a very important biomedical problem -- the results derived for HNM apply more broadly to other areas of causal inference.

\section{Background}
\subsection{Definitions}
 We define a causal process using a \textit{structural equation model} (SEM), or a series of equations in the form:
\begin{equation} \label{eq_SEM_gen}
    Z_i = f_i(\textnormal{Pa}(Z_i),E_i), \hspace{2mm}\forall Z_i \in \bm{Z},
\end{equation}
where $\textnormal{Pa}(Z_i)$ denotes the \textit{parents}, or direct causes, of $Z_i$. The set $\bm{E}$ contains mutually independent error terms. We assume $\mathbb{E}(\bm{E}) = 0$ without loss of generality. A linear SEM admits the more specific form:
\begin{equation} \label{eq_SEM_linear}
    Z_i = \textnormal{Pa}(Z_i) \beta_{\textnormal{Pa}(Z_i)Z_i} + E_i, \hspace{2mm}\forall Z_i \in \bm{Z},
\end{equation}
where $\beta$ denotes a matrix of coefficients. An SEM is \textit{invertible} if we can recover values of $\bm{E}$ uniquely from the values of $\bm{X}$.

 A \textit{directed graph} $\mathbb{G}$ is a graph with a directed edge $\rightarrow$ or $\leftarrow$ between any two vertices in $\bm{Z}$. We have $Z_i \rightarrow Z_j$ in $\mathbb{G}$ if $Z_i \in \textnormal{Pa}(Z_j)$ or, equivalently, $Z_j$ is a \textit{child} or direct effect of $Z_i$: $Z_j \in \textnormal{Ch}(Z_i)$. The \textit{neighbors} of $Z_i$ unify the parents and children: $\textnormal{Ne}(Z_i) = \textnormal{Pa}(Z_i) \cup \textnormal{Ch}(Z_i)$. A \textit{sink node} is a vertex without children. A \textit{directed path} from $Z_i$ to $Z_j$ refers to a sequence of adjacent directed edges from $Z_i$ to $Z_j$. $Z_i$ is an \textit{ancestor} of $Z_j$, denoted by $Z_i \in \textnormal{Anc}(Z_j)$, when there exists a directed path from $Z_i$ to $Z_j$; we likewise say $Z_j$ is a \textit{descendant} of $Z_i$. The set $\textnormal{Nd}(Z_i)$ corresponds to the non-descendants of $Z_i$. A \textit{cycle} occurs when $Z_i \in \textnormal{Anc}(Z_j)$, and we have $Z_j \rightarrow Z_i$. A directed graph is called a \textit{directed acylic graph} (DAG), if it does not contain cycles. An \textit{augmented graph} $\mathbb{G}^\prime$ is a DAG over $\bm{Z} \cup \bm{E}$ such that $E_i \in \textnormal{Pa}_{\mathbb{G}^\prime}(Z_i)$ and $\textnormal{Pa}_{\mathbb{G}^\prime}(E_i) = \emptyset$ for all $E_i \in \bm{E}$. We provide an example of a directed graph in Figure \ref{fig_root_cause} and its corresponding augmented graph in Figure \ref{fig_graphs}.
 
 The triple $\langle Z_i, Z_j, Z_k \rangle$ forms a \textit{collider} in $\mathbb{G}$, if we have $Z_i \rightarrow Z_j \leftarrow Z_k$, and $Z_i$ and $Z_k$ are non-adjacent. $Z_i$ and $Z_j$ are \textit{d-connected} given $\bm{W}  \subseteq \bm{Z} \setminus \{Z_i, Z_j\}$ if there exists a path between the two vertices such that any collider on the path is an ancestor of $\bm{W}$ and no non-collider on the path is in $\bm{W}$. Otherwise, $Z_i$ and $Z_j$ are \textit{d-separated} given $\bm{W}$. 
 
A density $p(\bm{Z})$ associated with a DAG $\mathbb{G}$ factorizes according to the product of the conditional densities of each variable in $\bm{Z}$ given its parents:
\begin{equation} \nonumber
p(\bm{Z})=\prod_{i=1}^{p} p(Z_i | \textnormal{Pa}(Z_i)).
\end{equation}
Any distribution which factorizes according to the above equation also satisfies the \textit{global Markov property} where d-separation between $Z_i$ and $Z_j$ given $\bm{W}$ in $\mathbb{G}$ implies conditional independence (CI) between $Z_i$ and $Z_j$ given $\bm{W}$ \cite{Lauritzen90}. We refer to the converse as \textit{d-separation faithfulness}, where CI implies d-separation. The density $p(\bm{X})$ is \textit{causally minimal} if no proper subset of $\mathbb{G}$ also obeys the global Markov property. D-separation faithfulness implies causal minimality \cite{Peters14}.

The \textit{Kolmogorov complexity} of a finite binary string $x$, denoted by $K(x)$, is the length of the shortest self-delimiting binary program that generates $x$ on a universal Turing machine and then halts. The universal Turing machine is not unique, but the Kolmogorov complexity between any two such machines only differs by at most a constant. Most equalities and inequalities in algorithmic information theory are therefore only understood up to a constant; the notation $\stackrel{+}{=}$ means equality up to a constant and likewise $\stackrel{+}{\leq}$ for inequality.

 To prevent cluttering of notation with too many parentheses, we write $p(Y)$ as $p_Y$ when referring to the entire density. We keep the standard notation $p(y) = p(Y=y)$ when referring to a specific value of the density.

 \subsection{Related Work}
Authors have proposed to identify causal direction using functional forms more restrictive than HNM. For example, LiNGAM considers a linear SEM with non-Gaussian errors, while the additive noise model (ANM) given by $Y=m(X) + \varepsilon$ considers a nonlinear SEM with additive noise \cite{Shimizu06,Hoyer08}. The post-nonlinear model (PNL) assumes that the error can be made homoscedastic under a monotonic transformation of the response \cite{Zhang09}. All of these models therefore only consider additive errors, whereas HNM allows both additive and multiplicative forms.

Recently, \cite{Xu22} also considered HNM and proposed an algorithm called HEC for determining causal direction in the bivariate case. HEC divides the range of the predictor variable into a finite set of bins and then fits an additive model in each bin. The authors additionally assume that the error terms follow a Gaussian distribution in order to optimize the number of bins using the BIC score. Another algorithm called Fourth Order Moment (FOM) assumes approximately Gaussian errors but allows the conditional variance to change in a smooth, rather than in a piece-wise, fashion \cite{Cai20}. GRCI in contrast admits a smooth conditional variance \textit{and} allows the error term to admit an arbitrary, potentially non-Gaussian distribution.

Other methods, such as those proposed in \cite{Tagasovska20,Mitrovic18,Liu17}, also allow heteroscedastic noise but  determine causal direction \textit{without} recovering the error terms. We therefore cannot use these algorithms to compute the Shapley values necessary for identifying patient-specific root causes of disease.

A third set of algorithms attempt to identify root causes rather than just determine causal direction. The RCI algorithm for example identifies patient-specific root causes of disease but assumes LiNGAM \cite{Strobl22}. Unfortunately, we cannot simply substitute LiNGAM with HNM in RCI because \textit{indirect} causal relations may not follow HNM -- i.e., HNM is not closed under marginalizatin. Other authors defined patient-specific root causes as conditional outliers, but not all root causes are outliers and not all outliers induce disease \cite{Janzing19}. We therefore instead define patient-specific root causes using Shapley values based on model predictivity. A third algorithm identifies root causes by quantifying changes in the marginal distribution of $D$ after substituting certain causal conditionals into an SEM, but this method struggles to scale beyond several variables and identifies root causes at the population level rather than at the desired patient-specific level \cite{Budhathoki21}. The root causes of complex diseases likely differ dramatically between patients, so we must identify \textit{patient-specific} root causes in order to make complex diseases tractable.

Both \cite{Janzing19,Budhathoki21} as well as the recent paper \cite{Budhathoki22} further assume knowledge of the causal graph. The authors carry out all of their experiments with known causal graphs. While the authors mention in passing that we can recover the error terms in invertible models, they do not address the hard problem of estimating the error term values from data without prior knowledge. The authors in \cite{Budhathoki21,Budhathoki22} also only utilize linear Gaussian models in their experiments rather than non-linear ANMs or even LiNGAM. Directly identifying the error term values without knowledge of the causal graph is critical in biomedical applications, where little to no prior knowledge may exist about the underlying causal relations. The investigators of \cite{Strobl22} solve this problem with the RCI algorithm but again only consider LiNGAM. GRCI in contrast recovers the error terms de novo under the flexible HNM class that subsumes both ANM and LiNGAM.

\begin{tcolorbox}[breakable,enhanced,frame hidden]
In summary, GRCI improves upon previous work because it:
\begin{enumerate}[leftmargin=*,label=(\arabic*)]
    \item directly recovers the error terms without prior knowledge of the causal graph;
    \item adopts the identifiable heteroscedastic noise model which includes both LiNGAM and ANM as special cases;
    \item generalizes RCI to models that are not closed under marginalization, such as HNM and ANM;
    \item identifies root causes at the patient-specific level in order to make complex diseases tractable.
\end{enumerate}
\end{tcolorbox}

\section{The Heteroscedastic Noise Model} \label{sec_HNM}

\subsection{Definition}
 We set $\bm{Z} = \bm{X} \cup D$, where $D$ denotes a binary diagnostic label. We will have more to say about $D$ in Section \ref{sec_root} and focus on $\bm{X}$ for now. We can generalize the linear SEM in Equation \eqref{eq_SEM_linear} to an HNM SEM as follows:
\begin{definition} (Heteroscedastic noise model) \label{assump_HNM}
An SEM obeys the heteroscedastic noise model (HNM) if the following holds for each $X_i \in \bm{X}$:
\begin{equation} \label{eq_HNM}
    X_i = m_i(\textnormal{Pa}(X_i)) + E_i  \sigma_i(\textnormal{Pa}(X_i)),
\end{equation}
for non-linear functions $m_i$ and $\sigma_i > 0$.
\end{definition}
\noindent We assume that $\mathbb{E}\big(\bm{E}\big) = 0$ and $\mathbb{E}\big(|\bm{E}|\big) = 1$ without loss of generality. HNM thus generalizes the linear SEM in Equation \eqref{eq_SEM_linear} by allowing the expectation and MAD (of the mean) to change as arbitrary non-linear functions of the parents. Further, HNM includes ANM as a special case where $\sigma_i$ is a constant. 

\begin{figure}
    \centering
    \includegraphics[scale=0.6]{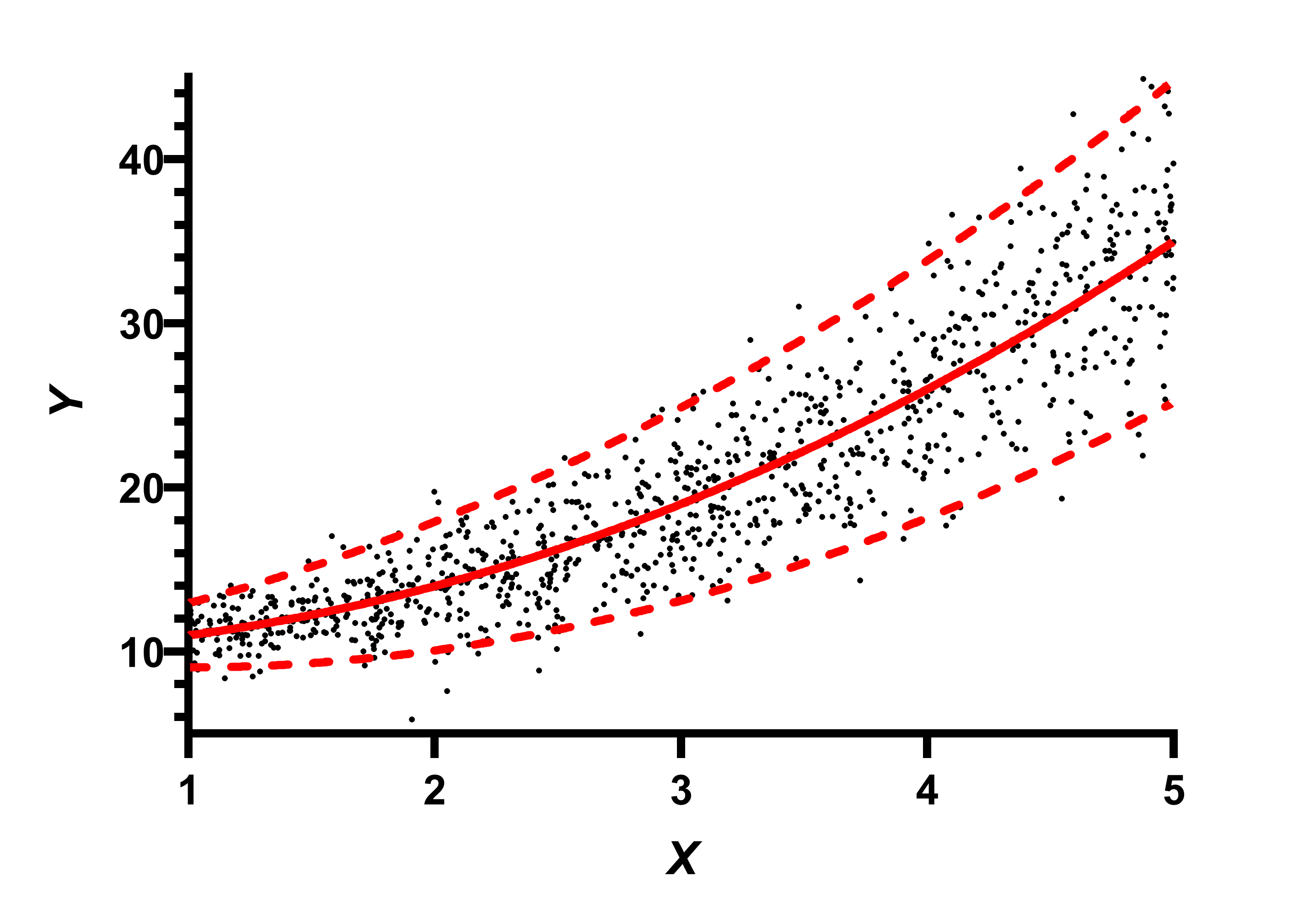}
    \caption{Example of an HNM with $Y=X^2 + EX$ plus 10.}
    \label{fig:HNM}
\end{figure}

Consider for example the bivariate HNM in Figure \ref{fig:HNM}. The conditional expectation in solid red and conditional MAD in dashed red (at 95\% prediction intervals) change as functions of $X$. In contrast, the best linear SEM erroneously fits a linear conditional expectation and assumes a constant variance. HNM thus increases modeling flexibility considerably.

Investigators have however proposed other models that generalize ANM in the literature. The post non-linear model (PNL) for example considers an invertible non-linear outer transformation $h_i$ such that $X_i = h_i(m_i(\textnormal{Pa}(X_i)) + E_i)$ for each $X_i \in \bm{X}$ \cite{Zhang09}. The conditional MAD depends on the conditional expectation in PNL, whereas HNM allows arbitrary changes of the MAD. Furthermore, optimizing PNL models requires non-convex procedures that frequently get stuck in local optima and overfit in practice \cite{Zhang09,Uemura20,Breiman85,Zhang15,Keropyan23}. Recovering the error terms in HNM in contrast involves solving two least squares problems as detailed in Section \ref{sec_HNM:extract}. HNM therefore offers additional control over the conditional MAD and admits easier learning procedures.

\subsection{Identifiability} \label{sec_identify}

The increased flexibility of accounting for heteroscedastic noise fortunately preserves \textit{identifiability} of the model, or the ability to pinpoint the exact DAG when given the joint distribution.

We assume strictly positive densities throughout. We first have the following result in the bivariate case:
\begin{theorem} \label{thm_DE}
Assume that:
\begin{itemize}
    \item the forward model $X \rightarrow Y$ obeys HNM so that $p(x,y) = p\big(\frac{y-m(x)}{\sigma(x)}\big) p(x)$ with $m(X)$ and $\sigma(X)$ once differentiable;
    \item there is a backward model $Y \rightarrow X$ also obeying HNM so that $p(x,y) = p\big(\frac{x-n(y)}{t(y)}\big) p(y)$.
\end{itemize}
Then the following differential equation holds:
\begin{equation} \label{eq_DE}
\begin{aligned}
     &-\frac{\sigma(x)}{Q(x,y)}\frac{\partial ^2}{\partial x \partial y}r(x,y)-\frac{\partial ^2}{\partial y^2}r(x,y) -\\ &\frac{\sigma^\prime(x)}{Q(x,y)}\frac{\partial }{\partial y}r(x,y)  = q^{\prime \prime}(y) + \frac{\sigma^\prime(x)}{Q(x,y)} q^\prime(y),
    \end{aligned}
\end{equation} 
where:
\begin{itemize}
\item $r(x,y) = \textnormal{log } p\big(\frac{x-n(y)}{t(y)}\big)$ and $ q(y) = \textnormal{log } p(y)$ both twice differentiable;
\item $Q(x,y) = \sigma(x)m^\prime(x) + (y-m(x))\sigma^\prime(x)$.
\end{itemize}
Moreover, if there exists a quadruple $(x_0,m(x_0),\sigma(x_0),p(x_0|y))$ such that $Q(x_0,y) \not = 0$ for all but countably many $y$, then $p_Y$ is completely determined by $(y_0, q^\prime(y_0))$ -- i.e., the set of all $p_Y$ satisfying the differential equation is contained in a two dimensional affine space.
\end{theorem}
\noindent We delegate the longer proofs to the Supplementary Materials.

Equation \eqref{eq_DE} expresses a very specific relationship and suggests that finding a backward model satisfying the relation is like finding a needle in the haystack; we will almost never encounter this needle in practice. The statement that $p_Y$ lies in a two dimensional space formalizes this intuition. It implies that the forward model cannot be inverted in general because the space of all possible $p_Y$ is infinite dimensional \textit{a priori}.

We recover the differential equation $-\frac{\sigma(x)}{Q(x,y)}\frac{\partial ^2}{\partial x \partial y}r(x,y)-\frac{\partial ^2}{\partial y^2}r(x,y)  = q^{\prime \prime}(y)$ in the special case of an additive noise model (ANM) with $\sigma(x)$ a constant and $\sigma^\prime(x) = 0$ -- thus replicating Lemma 1 in \cite{Janzing10_2}. We can see that this relation holds when $p_{XY}$ is Gaussian, a well-known case where we \textit{cannot} identify the causal direction. We can of course just work out the equations with $Y=X\beta + \varepsilon_Y$: $-\frac{\sigma(x)}{Q(x,y)}\frac{\partial ^2}{\partial x \partial y}r(x,y) \stackrel{+}{=} -\frac{1}{\beta}\frac{\beta}{\sigma^2_Y} = -\frac{1}{\sigma^2_Y}$, $\frac{\partial ^2}{\partial y^2}r(x,y) \stackrel{+}{=}  -\frac{1}{\sigma^2_Y}$ and $q^{\prime \prime}(y) = \frac{2}{\sigma^2_Y}$, so that Equation \eqref{eq_DE} holds in the Gaussian case. But more intuitively, Theorem 1 says that, if we are given information about $X$ in terms of $(\textcolor{blue}{x_0,x_0\beta},\textcolor{red}{\sigma_X},p(x_0|y))$, then we can recover $p_Y$ with two points $(\textcolor{Green}{y_0, q^\prime(y_0)})$ when HNM holds in both directions. This of course holds in the Gaussian case because we can recover the entire (centered) bivariate density by only knowing $(\textcolor{blue}{\beta},\textcolor{red}{\sigma_X}, \textcolor{Green}{\sigma_Y})$.

The fact that $p_Y$ is completely determined by just two parameters of $Y$ when both directions hold suggests that $p_{X|Y}$ provides a substantial amount of information about $p_Y$. This conflicts with past work postulating that nature implements an \textit{independence of causal mechanisms}, whereby $p_{X|Y}$ provides almost no information about $p_Y$ \cite{Janzing10,Janzing10_2}. Authors rigorously define this information as follows:
\begin{definition} (Algorithmic mutual information)
Let $s$ and $t$ denote two binary strings. The algorithmic mutual information between $s$ and $t$ is:
\begin{equation} \nonumber
 I(s:t) = K(t) - K(t|s^\star),
\end{equation}
where $s^\star$ denotes the shortest program that computes $s$.
\end{definition}
\noindent If $p_{X|Y}$ provides information about $p_Y$, then $K(p_Y|p^\star_{X|Y})$ is small, so we expect $I(p_Y:p_{X|Y}) \gg 0$. 

The two parameter conclusion from Theorem \ref{thm_DE} implies that $I(p_Y:p_{X|Y}) \gg 0$ in general under HNM. We can alternatively interpret Theorem \ref{thm_DE} as follows: we must choose $p_Y$ in a contrived fashion once we know $p_{X|Y}$, so that Equation \eqref{eq_DE} holds. The following theorem formalizes this intuition by showing that the complexity of $p_Y$ indeed lower bounds $I(p_Y:p_{X|Y})$; in other words, if $I(p_Y:p_{X|Y}) \gg 0$, then $p_Y$ is likely complex.
\begin{theorem} \label{thm:k_py}
Consider the same assumptions as Theorem \ref{thm_DE}. If both the forward and backward models follow HNM, then we have:
\begin{equation} \nonumber \label{eq_alg_MI}
\begin{aligned}
       &I(p_Y:p_{X|Y})\\ &\stackrel{+}{\geq} K(p_Y) - \inf_{(x_0,y_0)} K(x_0,m(x_0),\sigma(x_0),y_0,q^\prime(y_0)), 
\end{aligned}
\end{equation}
assuming of course that all inputs are computable.
\end{theorem}
The above theorem suggests that $p_Y$ likely has high Kolmogorov complexity because $I(p_Y:p_{X|Y}) \gg 0$. This conclusion also dovetails nicely with complexity based approaches which posit that $K(p_X) + K(p_{Y|X}) \stackrel{+}{\leq} K(p_Y) + K(p_{X|Y})$ when $X \rightarrow Y$ \cite{Janzing10,Stegle10}. If both the forward and backward directions admit HNM, then the inequality is still likely to hold because $K(p_Y)$ is large. Finally, Theorem \ref{thm:k_py} connects with the main idea of the Information Geometric Causal Inference (IGCI) algorithm, where we can determine the causal direction $X \rightarrow Y$ when we can replace $p_X$ with a simple density, such as the uniform or Gaussian density, but preserve the correlation between $p_X$ and an arbitrary property of $p_{Y|X}$ \cite{Janzing12,Janzing15}. GRCI will go a step further than IGCI by determining both causal direction \textit{and} the values of the error terms in order to compute patient-specific statistics.

GRCI will in particular extract the values of \textit{all} of the error terms by partialing out the parents of each variable in $\bm{X}$. The algorithm thus requires identifiability of the entire causal graph $\mathbb{G}$, but Theorem \ref{thm_DE} only applies to the bivariate case. We can fortunately extend Theorem \ref{thm_DE} to the multivariate setting by considering the following definition:
\begin{definition} (Restricted HNM) Equation \eqref{eq_HNM} is a restricted HNM if, for all $Y \in \bm{X}$, $X \in \textnormal{Pa}(Y)$ and $\bm{S}$ such that $(\textnormal{Pa}(Y) \setminus X) \subseteq \bm{S} \subseteq (\textnormal{Nd}(Y) \setminus X)$, there exists $\bm{S} = \bm{s}$ where $p(\bm{s})>0$
and $p(x,y|\bm{s})$ do \underline{not} satisfy Equation \eqref{eq_DE}.
\end{definition}
\noindent In other words, Equation \eqref{eq_DE} does not hold when we condition on some subset of the non-descendants of $Y$ not including a member of $\textnormal{Pa}(Y)$ -- notice that this is a very weak assumption. Let $\mathcal{G}$ denote the space of all causally minimal DAGs obeying a restricted HNM. We have the following result:
\begin{theorem} \label{thm:full}
Assume Equation \eqref{eq_HNM} is a restricted HNM according to $\mathbb{G}$. Then, $\mathbb{G}$ is uniquely identified from $\mathcal{G}$.
\end{theorem}
\noindent Our direct proof reduces the multivariate model to a bivariate one and then applies a contradiction using Theorem \ref{thm_DE}. We can also prove the statement indirectly using a general result in \cite[Theorem 2]{Peters11}. We conclude that the HNM model uniquely identifies the entire DAG as required for GRCI.

\subsection{Error-Term Extraction} \label{sec_HNM:extract}
 We can extract the error terms $\bm{E}$ from HNM using the Partial-Out algorithm summarized in Algorithm \ref{alg_PO}. The error term $E_i \in \bm{E}$ corresponds to:
\begin{equation} \label{eq_errHNM}
    E_i = \frac{X_i - m_i(\textnormal{Pa}(X_i))}{\sigma_i(\textnormal{Pa}(X_i))}.
\end{equation}
The call Partial-Out($\textnormal{Pa}(X_i),X_i$) first estimates the conditional expectation $m_i(\textnormal{Pa}(X_i))$ by regressing $X_i$ on $\textnormal{Pa}(X_i)$ in Line \ref{alg_PO:expectation}. Partial-Out optimizes all regression hyperparameters by cross-validation. The residuals correspond to:
\begin{equation} \nonumber
    X_i - \widehat{m}_i(\textnormal{Pa}(X_i)) = E_i \sigma_i(\textnormal{Pa}(X_i)) + o_p(1),
\end{equation}
where $\widehat{m}_i(\textnormal{Pa}(X_i))$ denotes the estimate of the conditional expectation using a non-linear regression method. Let $\ddot{m}_i(\textnormal{Pa}(X_i))$ denote the estimates of the conditional expectation on the validation folds with the best hyperparameter set. The algorithm then estimates the conditional MAD in Line \ref{alg_PO:variance} by regressing $|X_i - \ddot{m}_i(\textnormal{Pa}(X_i))|$ on $\textnormal{Pa}(X_i)$ via least squares (or mean squared error) using the \textit{same folds} as in Line \ref{alg_PO:expectation} because:
\begin{equation} \nonumber
\begin{aligned}
    &\mathbb{E}\big(|X_i - m_i(\textnormal{Pa}(X_i))|\big|\textnormal{Pa}(X_i)\big)\\ = \hspace{1mm} &\sigma_i(\textnormal{Pa}(X_i)) \cancelto{1}{\mathbb{E}\big(|E_i|\big|\textnormal{Pa}(X_i)\big)}.
\end{aligned}
\end{equation}
Care must be taken to regress $|X_i - \ddot{m}_i(\textnormal{Pa}(X_i))|$ and not $|X_i - \widehat{m}_i(\textnormal{Pa}(X_i))|$ so that the training folds from Line \ref{alg_PO:expectation} do not influence the validation folds in Line \ref{alg_PO:variance}. We use the conditional MAD instead of the conditional standard deviation because we can directly estimate the conditional MAD and divide by it. Squaring the residuals in Step \ref{alg_PO:variance} to estimate the conditional variance and then taking its square root to obtain the conditional standard deviation can lead to large estimation errors in practice. 

Partial-Out finally computes the error estimate in Line \ref{alg_PO:error} as:
\begin{equation} \label{eq_errHNM_est}
    \widehat{E}_i = \frac{X_i - \widehat{m}_i(\textnormal{Pa}(X_i))}{\widehat{\sigma}_i(\textnormal{Pa}(X_i))},
\end{equation}
per Equation \eqref{eq_errHNM}. We \textit{partial out} $\textnormal{Pa}(X_i)$ from $X_i$, when we compute $\widehat{E}_i$ by running Partial-Out($\textnormal{Pa}(X_i),X_i$) under HNM. 

We can implement the regressions with a variety of non-linear regression methods. We use linear splines in our experiments due to their relative robustness to overfitting and their ability to admit fast leave one out cross-validation using the Sherman–Morrison–Woodbury formula. We normalize all variables to $[0,1]$ and then use $m$ equispaced knots on $[0,1]$ (always including $1$ and replacing it with an offset). We choose $m$ by leave one out cross-validation from 10 equispaced points between $2$ and $\sqrt{n/10}$ inclusive, where $n$ denotes the sample size. We generalize to multivariate regression by randomly projecting $t>1$ variables onto $[0,1]$ using $\sum_{i=1}^t w_i X_i$ with the vector $\bm{w}$ obeying a Dirichlet distribution with alpha vector equal to all ones. This process ensures that we sample all weights uniformly from the $t-1$ simplex, since we have no prior knowledge about the sparsity level. 

\begin{algorithm}[t]
  \nonl \textbf{Input:} $\bm{V},X_i$\\
 \nonl \textbf{Output:}  $\widehat{E}_i$\\
 \BlankLine

$\widehat{m}_i(\bm{V}), \ddot{m}_i(\bm{V})\leftarrow$ Regress $X_i$ on $\bm{V}$ with cross-validation\\ \label{alg_PO:expectation}
$\widehat{\sigma}_i(\bm{V}) \leftarrow$ Regress $|X_i - \ddot{m}_i(\bm{V})|$ on $\bm{V}$ with cross-validation\\ \label{alg_PO:variance}
Compute $\widehat{E}_i$ per Equation  \eqref{eq_errHNM_est} \label{alg_PO:error}
\caption{Partial-Out} \label{alg_PO}
\end{algorithm}

\section{Patient-Specific Root Causes of Disease} \label{sec_root}

\begin{figure}
\centering 
\begin{tikzpicture}[scale=1.0, shorten >=1pt,auto,node distance=2.8cm, semithick]
                    
\tikzset{vertex/.style = {inner sep=0.4pt}}
\tikzset{edge/.style = {->,> = latex'}}
 
\node[vertex] (1) at  (0,0) {$X_1$};
\node[vertex] (2) at  (1.5,0) {$X_2$};
\node[vertex] (3) at  (3,0.5) {$X_3$};
\node[vertex] (4) at  (3,-0.5) {$X_4$};
\node[vertex] (5) at  (4.5,0.5) {$X_5$};
\node[vertex] (6) at  (4.5,-0.5) {$X_6$};
\node[vertex] (7) at  (6,0) {$D$};

\node[vertex] (8) at  (-0.5,1) {$E_1$};
\draw[edge] (8) to (1);
\node[vertex] (9) at  (1,1) {\textcolor{blue}{$E_2=e_2$}};
\draw[edge,blue] (9) to (2);
\node[vertex] (10) at  (2.5,1.5) {$E_3$};
\draw[edge] (10) to (3);
\node[vertex] (11) at  (4,1.5) {$E_5$};
\draw[edge] (11) to (5);
\node[vertex] (13) at  (2.5,-1.5) {$E_4$};
\draw[edge] (13) to (4);
\node[vertex] (14) at  (4,-1.5) {$E_6$};
\draw[edge] (14) to (6);
\node[vertex] (12) at  (5.5,1) {$E_7$};
\draw[edge] (12) to (7);

\draw[edge] (1) to (2);
\draw[edge,blue] (2) to (3);
\draw[edge,blue] (2) to (4);
\draw[edge,blue] (3) to (5);
\draw[edge,blue] (4) to (6);
\draw[edge,blue] (5) to (7);
\draw[edge,blue] (6) to (7);
\end{tikzpicture}

\caption{The augmented graph of Figure \ref{fig_root_cause}. We represent the shock as a perturbation to the exogenous error term $E_2$.} \label{fig_graphs}
\end{figure}
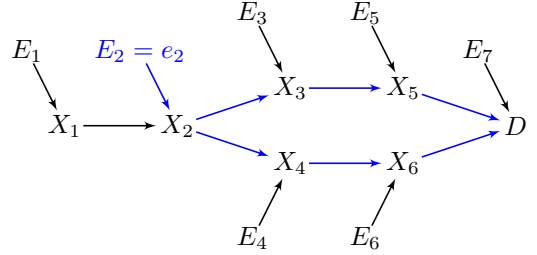

\subsection{Definition} \label{sec_root:def}
We want GRCI to compute patient-specific statistics to more specifically identify \textit{patient-specific root causes of disease}, so we need to rigorously define the term. Consider a binary variable $D$ denoting a diagnostic label of disease when $D=1$ and healthy when $D=0$. We assume that $D$ is a sink node in $\mathbb{G}$; this is a reasonable assumption because scientists who seek to identify the causes of $D$ frequently measure phenomena like transcriptomic levels or environmental exposures that are believed to precede the diagnosis in time.

A patient-specific root cause then corresponds to an exogenous shock to an otherwise healthy causal process that increases the probability that $D=1$ as a downstream effect; we provided an example in Figure \ref{fig_root_cause}. We model this initial shock as a change from a ``healthy'' value $E_i = \widetilde{e}_i$ to an ``unhealthy'' one $e_i$ while keeping the structural equations unchanged. We can interpret the change in value as a stochastic natural intervention; it is \textit{stochastic} because it represents a draw from $\mathbb{P}(E_i)$ independent of $\bm{E} \setminus E_i$, and it is a \textit{natural intervention} because the error terms have no parents. The change of value of $E_i$ affects downstream variables and ultimately increases the probability that $D=1$ (Figure \ref{fig_graphs}). 

We can quantify the change in probability of developing disease. We consider the following logarithmic odds:
\begin{equation}  \nonumber
    f(\bm{E}) = \textnormal{ln}\left(\frac{\mathbb{P}(D=1|\bm{E})}{
    \mathbb{P}(D=0|\bm{E})}\right).
\end{equation}
Let $v(\bm{W})$ denote the conditional expectation $\mathbb{E}(f(\bm{E})|\bm{W})$ where $\bm{W} \subseteq \bm{E} \setminus E_i$. We can measure the change in probability when intervening on $E_i \in \bm{E}$ with the following difference:
 \begin{equation} \label{eq:gamma}
 \gamma_{E_i \bm{W}} = v(E_i, \bm{W}) - v(\bm{W}).
 \end{equation}
We have $\gamma_{E_i \bm{W}}>0$ when $E_i$ increases the probability that $D=1$ given $\bm{W}$ because $v(E_i, \bm{W}) > v(\bm{W})$.

 Complex diseases may have multiple root causes that induce disease only when present in specific combinations. For example, a single genetic perturbation may not lead to cancerous growth, but multiple mutations often do. We therefore average over all possible combinations of $\bm{W} \subseteq \bm{E} \setminus E_i$ with $E_i$:
\begin{equation} \label{eq_shap}
 S_i = \frac{1}{p}\hspace{-13mm}\underbrace{\sum_{\bm{W} \subseteq (\bm{E} \setminus E_i)} \frac{1}{\binom{p-1} {|\bm{W}|}}}_{\textnormal{Average over all possible combinations of } \bm{E} \setminus E_i} \hspace{-12mm}\gamma_{E_i \cup \bm{W}}.
\end{equation}
An instantiation of the random variable $S_i$ is precisely the \textit{Shapley value} of \cite{Lundberg18} because we average over the differences of (conditionally expected) model predictions for all possible combinations of the errors. The Shapley value is \textit{sample-specific} because it depends on the values of $\bm{E}$ for each sample. 

The Shapley value satisfies the following three desiderata for each sample \cite{Lundberg18}:
\begin{enumerate}[leftmargin=*,label=(\arabic*)]
    \item Local accuracy: $\sum_{i=1}^p s_i = f(\bm{e}) - \mathbb{E} f(\bm{E})$;
    \item Missingness: if $E_i \not \in \bm{E}$, then $s_i = 0$;
    \item Consistency: $s_i^\prime \geq s_i$ for any two models $f^\prime$ and $f$ where $\gamma^\prime_{e_i\bm{w}} \geq \gamma_{e_i\bm{w}}$ for all $\bm{W} \subseteq \bm{E} \setminus E_i$.
\end{enumerate}
The first criterion ensures that the sum of the Shapley values remains invariant when ported to populations with different disease prevalence rates $\mathbb{E} f(\bm{E})$. The second criterion means that the error term of $D$ always has a Shapley value of zero. Finally, the third criterion says that the Shapley value only increases when the associated error term increases the probability of a patient developing disease (among all subsets). The Shapley value is in fact the \textit{only} value satisfying the local accuracy, missingness and consistency properties.

The above three desiderata and the corresponding uniqueness of the Shapley value justify the following definition of a patient-specific root cause:
\begin{definition} \label{def_root}
$X_i \in \bm{X}$ is a patient-specific root cause of disease ($D=1$) if $X_i \in \textnormal{Anc}_{\mathbb{G}^\prime}(D)$ and $S_i=s_i > 0$.
\end{definition}
\noindent In other words, $X_i \in \bm{X}$ is a patient-specific root cause if it is a cause of $D$ and its error predictably induces $D=1$, where predictivity is defined using the Shapley value $s_i$. A patient may have multiple root causes that lead to disease because $S_i$ is defined for each variable in $X_i \in \bm{X}$. We do not consider the case where $s_i \leq 0$ because $E_i=e_i$ does not increase the probability of disease in this case. We may finally consider a threshold greater than zero provided we have additional background knowledge regarding a clinically meaningful effect size.

Another paper quantifies the root causal contribution of outliers also using the Shapley values based on the error terms of a structural equation model \cite{Budhathoki22}. However, their Shapley values differ from ours in many regards. First, their Shapley values do not quantify the predictivity of developing disease, but rather the probability of encountering an event more extreme than the one observed. We ultimately want to eliminate disease regardless of symptom severity, so we focus on identify root causes of disease ($D=1$) for a given patient rather than root causes of having symptoms worse than a given patient. Second, if we interpret a patient as an outlier event and use a diagnostic cut-off score for each sample, then their root causal contribution measure loses sample specificity because we apply the same cut-off score to each sample. In contrast, our root causal contribution measure maintains sample-specificity even with the same cut-off applied to all samples. Third, computing their Shapley values requires prior knowledge about the ``normal'' distributions of the error terms which is difficult to determine in biomedical applications. The authors only use the empirical distribution of the error terms in their experiments. Under this selection, the total score of their Shapley values depends on the disease prevalence rate of a population -- even when the structural equations remain intact -- because their values obey $\sum_{i=1}^p s_i = \ddot{f}(\bm{e})$ for some outlier score $\ddot{f}$ that increases with greater prevalence. Fourth, the authors assume that the diagnosis corresponds to a noiseless cutoff score, even though the diagnosis is noisy because it may differ between diagnosticians in practice. We allow a noisy label. Fifth, we can leverage the mutual independence of $\bm{E}$ and existing fast algorithms to approximate our Shapley values, whereas theirs requires brute force iteration over all possible permutations or Monte Carlo sampling. Our root causal contribution measure thus uniquely targets disease, applies to specific samples, adjusts to the disease prevalence rate, allows noisy labels and admits efficient computation with error term distributions inferred directly from the data.

Finally, we justify our approach using an interventionist account, but we can also regard our interpretation of root causes of disease as a particular type of backtracking counterfactual. \cite{Von22} introduced the backtracking conditional distribution $\mathbb{P}(\bm{E}^*|\bm{E})$ that relates the error terms of a factual world $\bm{E}$ to those of a counterfactual world $\bm{E}^*$.  We can explain the value of $X_i$ with a potentially infinite number of different values of upstream error terms depending on the choice of the backtracking conditional. We restrict our attention to invertible SEMs and model the change in value of a patient-specific root cause $X_i$ to its error term -- as opposed to combinations of multiple upstream error terms.

\subsection{Generalized Root Causal Inference} \label{sec_GRCI}

We now detail the GRCI algorithm that recovers patient-specific root causes of disease from data. We summarize GRCI in Algorithm \ref{alg_GRCI}.

\begin{algorithm}[b]
 \nonl \textbf{Input:} $\bm{X}$, test set $\mathcal{T}$\\
 \nonl \textbf{Output:}  $\mathcal{S}$\\
 \BlankLine

$\widehat{\mathbb{G}} \leftarrow$ Skeleton-Stable($\bm{X}$) \label{alg_GRCI:PC} \\
$\bm{E},\bm{N} \leftarrow$ Extract-Errors($\bm{X},\widehat{\mathbb{G}}$) \label{alg_GRCI:error}\\
Compute the matrix $\mathcal{S}$ containing the estimated Shapley values of each patient in $\mathcal{T}$ \label{alg_GRCI:S}

\caption{Generalized Root Causal Inference (GRCI)} \label{alg_GRCI}
\end{algorithm}

\subsubsection{Skeleton Discovery}

Non-linear regressors can easily overfit in high dimensions. GRCI therefore first reduces the dimensionality of the necessary regressions in Step \ref{alg_GRCI:PC} by identifying the \textit{skeleton} of $\bm{X}$, or the presence and absence of the directed edges in $\mathbb{G}$. GRCI uses an algorithm called Skeleton-Stable -- the skeleton discovery procedure of the well-known PC-Stable algorithm which identifies the skeleton using a series of CI tests \cite{Colombo14}. An edge is not present between any two variables $X_i$ and $X_j$ in $\mathbb{G}$ if and only if $X_i \ci X_j | \bm{W}$ for some $\bm{W} \subseteq \textnormal{Pa}(X_i)$ or some $\bm{W} \subseteq \textnormal{Pa}(X_j)$, under d-separation faithfulness \cite{Spirtes00}. Skeleton-Stable therefore tests whether $X_i$ and $X_j$ are conditionally independent given dynamically adjusted supersets of the parents. We skip further details of the algorithm, since they are not important for this paper.

\subsubsection{Global Error Term Extraction}
In Step \ref{alg_GRCI:error}, GRCI uses the skeleton identified by Skeleton-Stable to extract the error terms of $\bm{X}$ with the Extract-Errors algorithm.

We have summarized Extract-Errors in Algorithm \ref{alg_EE}. Extract-Errors initializes $\bm{M}$ to the set of all variables in $\bm{X}$. The algorithm then iteratively removes a member from $\bm{M}$ in Line \ref{alg_EE:remove} and places it into $\bm{N}$ in Line \ref{alg_EE:add} so that $\bm{N}$ ultimately contains a reverse partial-order of $\bm{X}$. Extract-Errors identifies the variable to remove from $\bm{M}$ in Step \ref{alg_EE:sink} using the Find-Sink algorithm.

We summarize Find-Sink in Algorithm \ref{alg_sink}. Find-Sink identifies the variable whose parents are most independent of its residuals due to the following result:
\begin{lemma} \label{lem_base}
If $X_i \in \bm{M}$ is a sink node, then $E_i \ci X_j$ for all $X_j \in \bm{M} \setminus X_i$. 
\end{lemma}
\begin{proof}
$E_i$ and $X_j$ are d-separated in $\mathbb{G}^\prime$ for all $X_j \in \bm{M} \setminus X_i$. The conclusion follows by the global Markov property. 
\end{proof}
The algorithm in particular runs Partial-Out on each variable $X_i \in \bm{M}$ given its neighbors to recover the residuals $R_i$. Find-Sink then computes the mutual information score $\max_{X_j \in \textnormal{Ne}(X_i)} I(X_j;R_i)$ using the nearest neighbor technique proposed in \cite{Kraskov04}. A lower mutual information score indicates a higher degree of independence. Find-Sink logs all of the mutual information scores associated with $\bm{M}$ in $\bm{T}$, and then identifies the sink node in Line \ref{alg_sink:min} as the variable in $\bm{M}$ associated with the smallest score in $\bm{T}$. 

Extract-Errors then partials out the sink node $S$ identified by Find-Sink in Line \ref{alg_EE:partial}. The algorithm also removes edges adjacent to $S$ in $\widehat{\mathbb{G}}$. The neighborhoods of some of the variables in $\bm{M}$ change due to this step -- denote these variables in $\bm{M}$ by $\bm{U}$. Extract-Errors updates the scores in $\bm{T}$ for $\bm{U}$ in the next iteration. Repeating this process of identifying a sink node in $\bm{M}$, partialing out its errors and placing it into $\bm{N}$ until $\bm{M}$ is empty results in (1) a reverse partial order in $\bm{N}$ and (2) all of the error terms collected in $\bm{E}$. More formally:
\begin{lemma} \label{lem_errors}
Extract-Errors recovers all of the error terms of $\bm{X}$.
\end{lemma}
\begin{proof}
We prove this induction. The base case follows by Lemma \ref{lem_base}. For the induction step, assume that Extract-Errors recovers all error terms when $|\bm{M}| = n$. We need to show that the statement holds for $n+1$. We can recover a sink node from $\bm{M}$ when $|\bm{M}| = n+1$ by Lemma \ref{lem_base}. The conclusion follows by the inductive hypothesis. 
\end{proof}

\subsubsection{Shapley Values}
GRCI finally computes the matrix $\mathcal{S}$ containing the Shapley values in Step \ref{alg_GRCI:S}. The $i^\textnormal{th}$ column and $j^\textnormal{th}$ row of $\mathcal{S}$ contains the Shapley value of $X_i \in \bm{X}$ for the patient $j$ in test set $\mathcal{T}$. We approximate these values to high accuracy in practice by predicting $D$ with XGBoost using the error terms recovered by Extract-Errors and then applying the TreeSHAP algorithm \cite{Lundberg18,Chen16}. We certify GRCI with the following theorem, where we assume access to a Shapley oracle that outputs the true Shapley values with mutually independent predictors:
\begin{theorem}
(Fisher consistency) Assume access to CI, regression and Shapley oracles. Then, under d-separation faithfulness and HNM over $\bm{X}$, GRCI recovers the true Shapley values and therefore the patient-specific root causes of disease for all samples in $\mathcal{T}$.
\end{theorem}
\begin{proof}
Skeleton-stable recovers a superset of the skeleton of $\mathbb{G}$ under d-separation faithfulness \cite{Colombo14}. Extract-Errors recovers all of the error terms of $\bm{X}$ by Lemma \ref{lem_errors}. The Shapley oracle now has access to the mutually independent error terms and can therefore compute the matrix $\mathcal{S}$ containing the true Shapley values \cite{Lundberg18}. The conclusion follows for the entries in $\mathcal{S}$ greater than zero by Definition \ref{def_root}. 
\end{proof}

\subsubsection{Time Complexity} \label{sec:time}
GRCI is composed of three steps as summarized in Algorithm \ref{alg_GRCI}. Skeleton-Stable in Step \ref{alg_GRCI:PC} calls a CI test at most $O(p^r)$ times, where $r$ denotes the maximum number of neighbors of a vertex in $\mathbb{G}$. The CI test we implement performs a fixed number of multivariate adaptive spline regressions (MARS) each requiring $O(nrm^4)$ time, where $n$ denotes the sample size and $m$ the maximum number of basis functions \cite{Friedman91}.\footnote{We replace random Fourier regression with MARS regression and use a fixed number of non-linear transformations similar to \cite{Strobl19}.} Skeleton-Stable therefore requires $O(nrp^rm^4)$ time. The Extract-Errors function in Step \ref{alg_GRCI:error} iterates twice over the variables, so it requires on the order of $p^2$ iterations. Each iteration is dominated by Partial-Out which requires $O(n^2b + b^3)$ time, where $b$ denotes the maximum number of basis functions used during cross-validation. Extract-Errors therefore requires $O(n^2p^2b + p^2b^3)$ time. Finally, the TreeSHAP algorithm computes in $O(ntld^2)$ time, where $t$ refers to the number of trees, $l$ the maximum number of leaves, and $d$ the maximum tree depth. Repeating this process for each of the $p$ variables requires $O(nptld^2)$ time. We emphasize that the mutual independence on $\bm{E}$ and our particular definition of the Shapley value with model-based conditional expectations in Equation \eqref{eq_shap} enables fast Shapley value computations with tree models. GRCI ultimately requires $O(nrp^rm^4) + O(n^2p^2b + p^2b^3) + O(nptld^2)$ time; we can adjust the first, second and third terms, if we use a different CI testing, Partial-Out regression or Shapley value computation procedure, respectively. We conclude that GRCI scales quadratically with respect to sample size and polynomially $O(p^r)$ with respect to the number of variables if $r \geq 2$.

\begin{algorithm}[t]
  \nonl \textbf{Input:} $\bm{X}, \widehat{\mathbb{G}}$\\
 \nonl \textbf{Output:}  $\bm{E},\bm{N}$\\
 \BlankLine

$\bm{M}, \bm{U} \leftarrow \bm{X}$\\
$\bm{N} \leftarrow \emptyset$\\
$\bm{T} \leftarrow \infty$\\
\Repeat{$\bm{M} = \emptyset$}{
      $S \leftarrow$ Find-Sink($\bm{M},\bm{U},\bm{T},\widehat{\mathbb{G}}$) \\ \label{alg_EE:sink}
      $\bm{M} \leftarrow \bm{M} \setminus S$ \label{alg_EE:remove}\\
      $\bm{N} \leftarrow \bm{N} \cup S$ \label{alg_EE:add} \\
      $\bm{E}_S \leftarrow \textnormal{Partial-Out}(\textnormal{Ne}_{\widehat{\mathbb{G}}}(S),S)$\label{alg_EE:partial}\\
      $\bm{U} \leftarrow $ members of $\bm{M}$ adjacent to $S$ in $\widehat{\mathbb{G}}$\\
      Remove edges adjacent to $S$ in $\widehat{\mathbb{G}}$
    }
\caption{Extract-Errors} \label{alg_EE}
\end{algorithm}

\begin{algorithm}[t]
 \nonl \textbf{Input:} $\bm{M},\bm{U},\bm{T},\widehat{\mathbb{G}}$\\
 \nonl \textbf{Output:} sink $S$\\
 \BlankLine

\textbf{return} $\bm{M}$ if $|\bm{M}| = 1$\\

\For{$X_i \in \bm{U}$}{
    $R_i \leftarrow \textnormal{Partial-Out}(\textnormal{Ne}_{\widehat{\mathbb{G}}}(X_i),X_i)$\\
    $T_i \leftarrow \max_{X_j \in \textnormal{Ne}(X_i)} I(X_j;R_i)$ \label{alg_sink:D}
}
$S \leftarrow \bm{M}[\argmin_{X_i \in \bm{U}} T_i]$ \label{alg_sink:min}
\caption{Find-Sink} \label{alg_sink}
\end{algorithm}

\section{Experiments}

\noindent\textbf{Hyperparameters. } GRCI requires two hyperparameters: the $\alpha$ value for Skeleton-Stable and the $k$ value for the nearest neighbor mutual information estimator. We set $\alpha$ to the liberal threshold of 0.1, which in practice causes Skeleton-Stable to output a superset of the true skeleton. This superset represents a small subset of the fully connected graph that greatly reduces the dimensionality of the regressions performed in Step \ref{alg_GRCI:error}.

We fixed $k=10$ for the mutual information estimator for three reasons. First, the entropy estimate is consistent for any fixed value of $k$. The standard deviation of the estimator also stabilizes at $k=10$ for most sample sizes according to Figure 4 of \cite{Kraskov04}. Moreover, the estimate is near exact when independence truly holds as shown in Figure 2 of \cite{Kraskov04}. Both of these experimental results hold for nearly all cases tested by the authors.\\${}$\vspace{-3mm}\\
\noindent \textbf{Reproducibility. } All R code needed to replicate experimental results is available at github.com/ericstrobl/GRCI.

\subsection{Causal Direction}

GRCI computes a (reverse) partial ordering $\bm{N}$, so we can use the algorithm to recover causal direction in the bivariate setting after assuming that an edge exists between $X$ and $Y$. We compared GRCI against four algorithms on their ability to identify causal direction in the bivariate setting:
\begin{enumerate}[label=(\arabic*)]
\item HEteroscedastic noise Causal model (HEC): bins $X$ and fits a polynomial regressor in each bin while assuming intra-bin homoscedasticity. The algorithm chooses the causal direction as the one minimizing the BIC score \cite{Xu22}.
\item Fourth Order Moment (FOM): estimates the fourth-order moment of the residuals using a heteroscedastic Gaussian process. The algorithm chooses the causal direction as the one minimizing the fourth-order moment \cite{Cai20}.
\item REgression and Subsequent Independence Test (RESIT): assumes an ANM, regresses out the conditional mean using a Gaussian process and determines causal direction using a reproducing kernel-based conditional independence test \cite{Peters14}.
\item Direct LiNGAM (DL): assumes variables are linearly related with non-Gaussian errors \cite{Shimizu11}. The algorithm decides causal direction using the differential entropy measure proposed in \cite{Hyvarinen13}.
\end{enumerate}
The first two algorithms cover state of the art methods that handle heteroscedastic noise. The other two algorithms are state of the art for the additive noise and linear non-Gaussian acyclic models. Other algorithms in the literature utilize information theoretic measures and do not impose functional forms. We however only compare against methods which can extract the values of the error terms, since we are ultimately interested in performing patient-specific root causal inference rather than just determining causal direction. 

\subsubsection{Synthetic Data} \label{sec:pair_synth}

We generated data using four different functional models:
\begin{enumerate}[label=(\arabic*)]
    \item LiNGAM: $Y=X\beta + E$
    \item ANM: $Y=f(X) + E$
    \item HNM: $Y=f(X) + E g(X)$,
    \item PNL: $Y=h(f(X) + E)$,
\end{enumerate}
with $f(X)$ and $g(X)-1$ uniformly sampled from the set $\{ \sqrt{X^2 + 1} -1 , X\Psi(X), 1/(1+\textnormal{exp}(-X))\}$; we subtracted one from $g(X)$ to ensure non-zero variance. We uniformly sampled $h(\cdot)$ for PNL from a set of strictly monotonic functions: $\{ \textnormal{tanh}(\cdot), \textnormal{ln}(1+\textnormal{exp}(\cdot)), 1/(1+\textnormal{exp}(-\cdot))\}$. We sampled the distribution of $E$ uniformly from the following possibilities: uniform distribution on $[-1,1]$, t-distribution with five degrees of freedom, chi square distribution with three degrees of freedom. Note that LiNGAM requires at least one non-Gaussian error, whereas ANM and HNM do not. We therefore also included the centered Gaussian distribution with variance $1/9$ as one of the possibilities for the error term of ANM and HNM. We repeated the above procedure 200 times for LiNGAM with non-Gaussian errors, 200 times for ANM with non-Gaussian errors and another 200 times with Gaussian errors, 200 times for HNM with non-Gaussian errors and another 200 times with Gaussian errors. We therefore generated a total of 1000 independent datasets.

\begin{figure}
    \centering
    \includegraphics[scale=0.8]{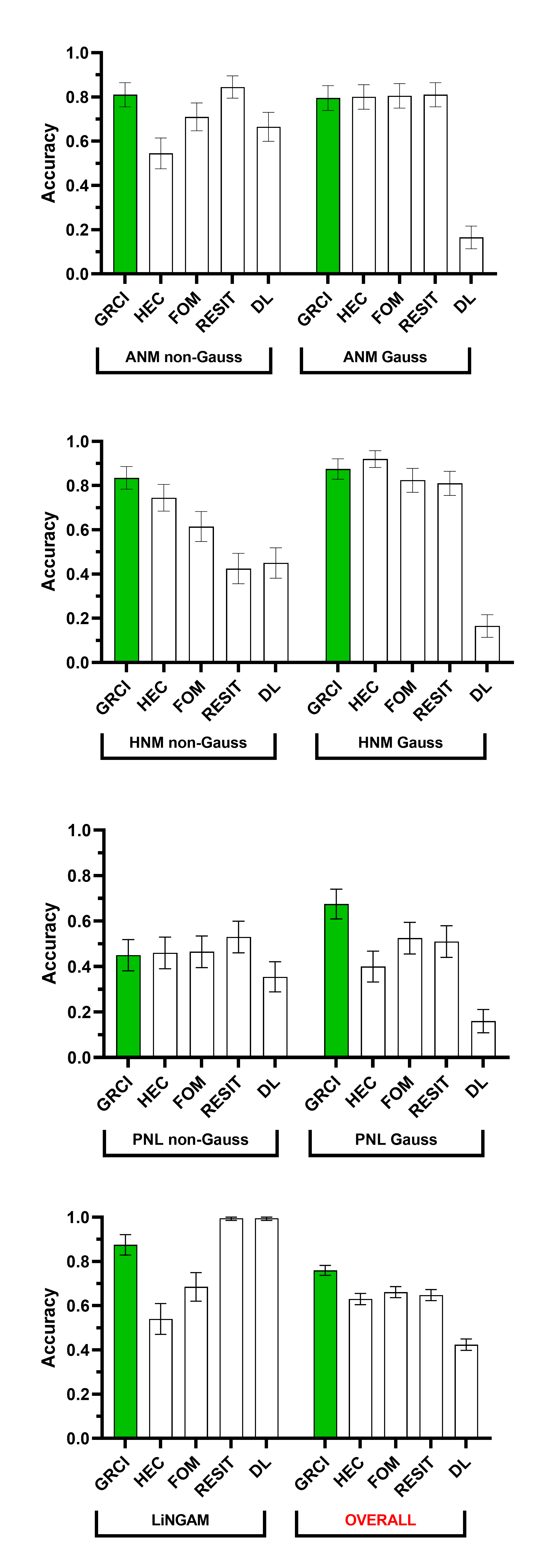}
    \caption{Results on causal direction with synthetic data under different conditions. Rows correspond to functional model and columns to Gaussianity. Error bars denote 95\% confidence intervals of the mean. GRCI performs relatively well across all conditions while the other algorithms only perform well in some cases.}
    \label{fig:synth_pair}
\end{figure}

We report the results in Figure \ref{fig:synth_pair}. LiNGAM, GRCI and RESIT performed well under LiNGAM. LiNGAM and RESIT outperformed GRCI in this case because they are specifically designed for the homoscedastic setting. Only GRCI and RESIT performed well under ANM with non-Gaussian errors because LiNGAM assumes linear conditional expectations. GRCI, HEC and FOM all performed equivalently with Gaussian error terms under both ANM and HNM. However, GRCI outperformed the other two -- sometimes by a very large margin -- with non-Gaussian errors. Recall that HEC and FOM make a variety of Gaussian approximations which unfortunately do not work well in the non-Gaussian setting. All algorithms performed poorly under PNL, but GRCI outperformed the others with Gaussian errors. Overall, GRCI achieved the best performance when averaged across all conditions. We conclude that GRCI maintains good performance across LiNGAM, ANM and HNM while other algorithms only perform well in special cases. Timing results are located in the Supplementary Materials; GRCI completed within 0.4 seconds on average.

\subsubsection{Real Data}

The T\"{u}bingen cause-effect pairs benchmark contains 108 datasets of real cause-effect pairs \cite{Mooij16}. We summarize the results for the 108 pairs in Figure \ref{fig:Tuebin}. As is standard in the literature, we exclude pairs containing multivariate vectors or binary variables; this includes pair numbers 47, 52-55, 70, 71, 105 and 107. We evaluate accuracy using the suggested weighted average in order to account for the potential bias introduced by pairs derived from the same multivariable dataset. The x-axis in Figure \ref{fig:Tuebin} corresponds to the cause-effect pair number (1-108), and the y-axis to the moving weighted accuracy. An ideal algorithm should achieve the highest weighted accuracy at any pair number. GRCI obtained an overall weighted accuracy of 81.6\%, as opposed to 71.2\% for FOM, 70.5\% for HEC, 64.0\% for RESIT and 51.5\% for LiNGAM. GRCI also maintained the best weighted accuracy at any pair number.  We conclude that GRCI accurately discovers causal direction using real data. In general, algorithms that account for heteroscedasticity (GRCI, FOM, HEC) perform better than those that only account for homoscedasticity (RESIT, LiNGAM), and algorithms that account for non-linear relations (GRCI, FOM, HEC, RESIT) perform better than those that only account for linear relations (LiNGAM). Timing results are located in the Supplementary Materials; GRCI completed within 5 seconds on average.

\begin{figure}
    \centering
    \includegraphics[scale=0.65]{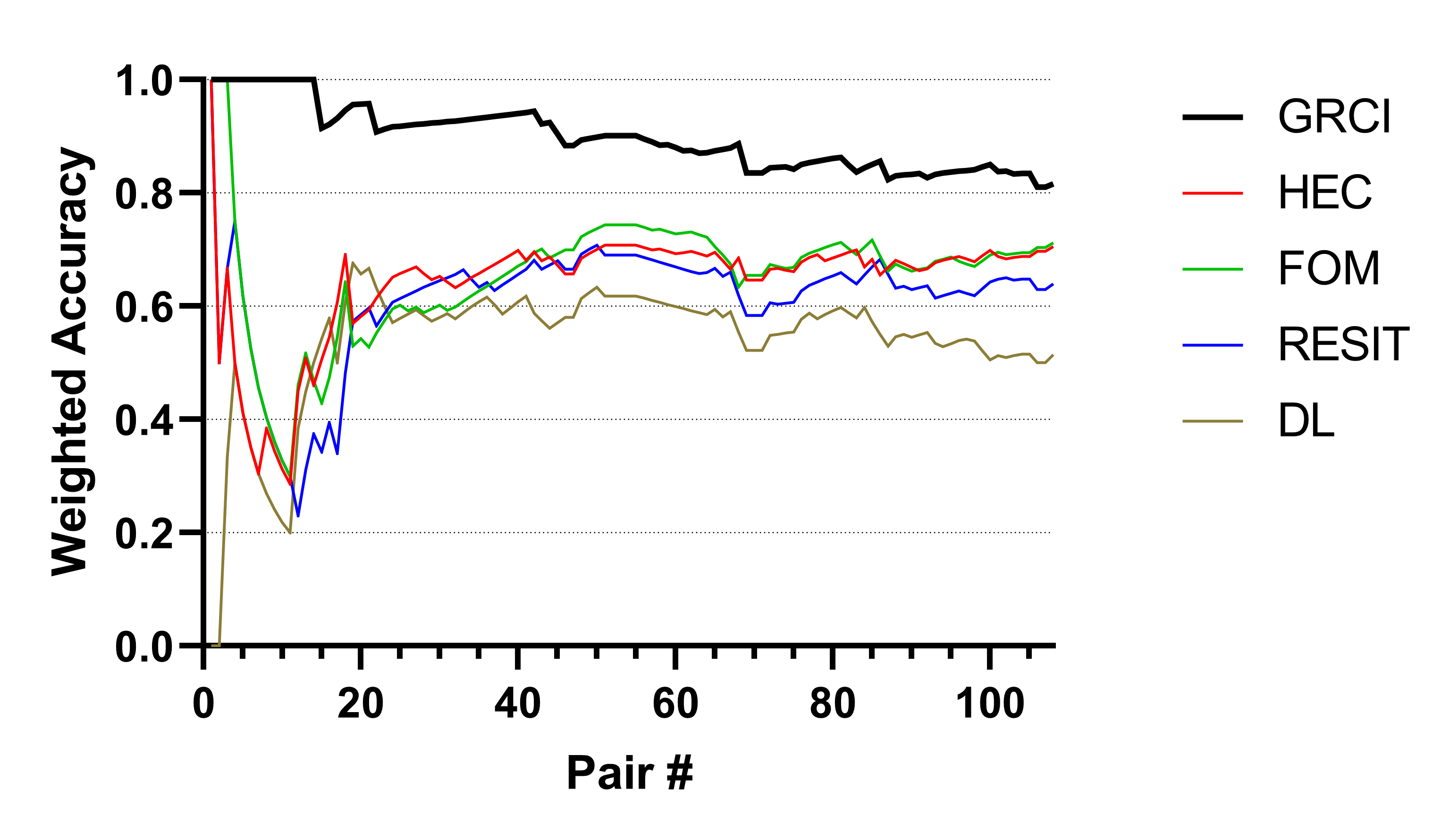}
    \caption{Results on the 108 T\"{u}bingen cause-effect pairs. GRCI maintains the highest weighted accuracy at any pair number.}
    \label{fig:Tuebin}
\end{figure}

\begin{tcolorbox}[breakable,enhanced,frame hidden]
Summarizing the results of the causal direction experiments:
\begin{enumerate}[leftmargin=*,label=(\arabic*)]
    \item GRCI maintains good performance across LiNGAM, ANM and HNM, regardless of whether the errors are Gaussian or not.
    \item HEC and FOM do not perform well when error terms deviate from Gaussianity.
    \item GRCI obtains and maintains the best performance with real data. 
\end{enumerate}
\end{tcolorbox}

\subsection{Root Causal Inference}

We next investigate the performance of GRCI in discovering patient-specific root causes of disease. We compare against four other algorithms:
\begin{enumerate}[label=(\arabic*)]
    \item Root Causal Inference (RCI): recovers patient-specific root causes assuming that the joint distribution obeys LiNGAM \cite{Strobl22}. 
    \item Prediction with ICA (ICA): runs ICA and then ranks the identified sources using a local variable importance measure of random forest \cite{Lasko19}. 
    \item Conditional Outliers (CO): learns a causal graph $\widehat{\mathbb{G}}$ and then identifies patient-specific root causes as conditional outliers according to the score $\frac{|X_i-m_i(\textnormal{Pa}_{\widehat{\mathbb{G}}}(X_i))|}{\sigma_i(\textnormal{Pa}_{\widehat{\mathbb{G}}}(X_i))}$ \cite{Janzing19}. 
    \item Model Substitution (MS): learns a causal graph and then identifies root causes of changes in the marginal distribution of $D$ by substituting causal conditional distributions into the joint distribution \cite{Budhathoki21}. 
\end{enumerate}
MS does not output sample-specific values, but we can still apply the population level values to each individual sample. 

GRCI, RCI and ICA recover the error terms directly without fully estimating the underlying DAG. However, CO and MS require a method for estimating the DAG. We tested RESIT and GDS as proposed in \cite{Peters14}, but they did not scale even after substituting a fast non-parametric conditional independence test \cite{Strobl19}. We therefore instead ran Steps \ref{alg_GRCI:PC} and \ref{alg_GRCI:error} of GRCI to recover a partial order. The parents of a variable must precede it in the partial order. We next ran Skeleton-Stable with conditioning sets restricted to preceding variables according to the partial order and then oriented directed edges according to the partial order. This process recovers a unique DAG. We finally ran Partial-Out to recover the error terms using the estimated DAG in order to ensure that both CO and MS also utilize HNM. We fixed the alpha threshold to 0.05 because it led to the best results in our experiments. 

Computing the ground truth Shapley values requires an exponential number of summations per Equation \eqref{eq_shap}. We therefore instead estimated the ground truth to negligible error by (1) feeding XGBoost fifty thousand samples of the \textit{ground truth} error terms and (2) running the TreeSHAP algorithm on the learned model. We reran all applicable algorithms (RCI and ICA) using XGBoost and TreeSHAP in order to prevent GRCI from achieving an unfair advantage due to possible biases introduced during ground truth estimation.

\subsubsection{Synthetic Data}

We generated data from a DAG with an expected neighborhood size of two, $\mathbb{E}(N)=2$. We assigned adjacencies using independent realizations of a Bernoulli$\Big(\frac{\mathbb{E}(N)}{p-1}\Big)$ random variable in an upper triangular matrix. We then replaced the binary variables twice with samples from $\textnormal{Uniform}([-1,-0.25] \cup [0.25, 1])$. Let $\beta^1$ denote the first resultant coefficient matrix, and $\beta^1_{ji}$ to the $j^\textnormal{th}$ row and $i^\textnormal{th}$ column; likewise for $\beta^2$. We generated the non-Gaussian error terms using the same procedure described in Section \ref{sec:pair_synth}. The HNM model corresponds to $X_i=f_i(\sum_{X_j \in \textnormal{Pa}(X_i)} X_j\beta^1_{ji}) + E_i g_i(\sum_{X_j \in \textnormal{Pa}(X_i)} X_j\beta^2_{ji})$ for each $X_i \in \bm{X}$ with functions $f_i,g_i$ drawn randomly as in Section \ref{sec:pair_synth}. We finally permuted the variable order. Repeating the above procedure 200 times for sample sizes of $n=500, 1000, 2000$ and dimensions $p=10, 30, 50$ generated a total of $200 \times 3 \times 3 = 1800$ datasets.

\textbf{Metrics. } Comparing the algorithms is not straightforward because the algorithms have different outputs. GRCI returns Shapley values for all of the variables. RCI returns Shapley values only for some of the variables, since it performs variable selection. ICA outputs sample-specific scores according to a random forest metric, but it can be modified to return Shapley values for all of the variables. MS outputs population level Shapley values, and CO outputs sample specific conditional outlier scores both only for some of the variables. We need a method that compares the algorithms on a common footing and accounts for outputs of different lengths. 

All algorithms fortunately can return a ranked list of variables. The top ranked variables ideally should correspond to the root causes with the largest effect on $D$. We therefore evaluated the algorithms using rank-biased overlap (RBO) \cite{Webber10}, a well-established metric that compares two ranked lists. Let $\mathcal{R}^k$, correspond to the ground truth ranking of the root causes for patient $k$ according to the true Shapley values. Similarly let $\widehat{\mathcal{R}}^k$ denote the estimate of the ranking given by an algorithm. The RBO corresponds to:
\begin{equation} \label{eq:RBO_shapley}
    \frac{1}{n} \sum_{k=1}^n \sum_{i=1}^{q_k} \widetilde{s}_i^k | \widehat{\mathcal{R}}_{1:i}^k \cap \mathcal{R}_{1:i}^k|/i,
\end{equation}
where $s_i^k$ denotes the true Shapley value of $X_i$ for patient $k$, $\widetilde{s}_i^k = \frac{s_i^k}{\sum_{i=1}^{q_k} s_i^k}$ the version normalized to sum to one, and $q_k$ the total number of root causes for patient $k$. RBO can compare ranked lists of potentially varying lengths and weighs top variables more heavily than bottom ones.
The metric takes values between zero and one; it equals one when the top ranked variables coincide exactly between the two lists, and zero when there is no overlap. A higher RBO is therefore better.

We focus primarily on the RBO because the algorithms output different variable importance measures. However, we also compute the mean squared error (MSE) to the proposed ground truth Shapley values as a secondary measure:
\begin{equation} \nonumber
    \frac{1}{nw} \sum_{k=1}^n \sum_{i=1}^p (\widehat{s}_i^k - s_i^k)^2.
\end{equation}
We set $\widehat{s}_i^k = 0$, if an algorithm does not output a score for variable $i$. An MSE of zero implies an RBO of one, but an algorithm can achieve a high RBO with a large MSE.

\textbf{Results. } We summarize the accuracy results with the synthetic data using RBO and MSE in Tables \ref{exp_synth:RBO} and \ref{exp_synth:MSE}, respectively. Recall that we implemented two versions of RCI and ICA - the original ones and the modified forms using TreeSHAP as labeled using the subscript $t$. We therefore compared GRCI against a total of six algorithms. Bolded values in each row of the tables correspond to the best performing algorithms according to paired two-tailed t-tests each at a Bonferonni corrected threshold of 0.05/6. 

\begin{table}[t]
\begin{subtable}{0.45\textwidth}  
\centering
\captionsetup{justification=centering,margin=2cm}
\begin{tabular}{cc|ccccccc}
\hhline{=========}
\textit{p} & \textit{n} & GRCI            & RCI   & RCI$_t$    & ICA & ICA$_t$ &CO    & MS            \\ \hline
10      & 500         & \textbf{0.735} & \textbf{0.706} & 0.690 & 0.579 & 0.639 & 0.616 &  0.508\\
     & 1000     & \textbf{0.773} & 0.699 & 0.689 & 0.603 & 0.669 & 0.623 & 0.502\\
     & 2000       & \textbf{0.809} & 0.710 & 0.708 & 0.614 & 0.695  & 0.631 & 0.503\\ \hline
30      & 500         & \textbf{0.653} & \textbf{0.622} & 0.616 & 0.477 & 0.519   & 0.496 & 0.392        \\
     & 1000         & \textbf{0.711} & 0.654 & 0.647 & 0.537 & 0.593    & 0.463 & 0.347\\
     & 2000         & \textbf{0.745} & 0.682 & 0.673 & 0.573 & 0.641   & 0.485 &  0.379\\ \hline
50      & 500         & \textbf{0.639} & 0.569 & 0.580 & 0.327 & 0.345    & 0.432 &  0.348      \\
     & 1000         & \textbf{0.685} & 0.613 & 0.609 & 0.506 & 0.556   & 0.402 & 0.338\\
     & 2000         & \textbf{0.741} & 0.642 & 0.636 & 0.555 & 0.615   & 0.383 &  0.311\\ 
\hhline{=========}
\end{tabular}
\caption{RBO} \label{exp_synth:RBO}
\end{subtable}

\vspace{5mm}\begin{subtable}{0.45\textwidth}  
\centering
\captionsetup{justification=centering,margin=2cm}
\begin{tabular}{cc|ccccccc}
\hhline{=========}
\textit{p} & \textit{n} & GRCI            & RCI   & ICA            \\ \hline
10      & 500         & \textbf{0.160} & 0.650 & 3.044  \\
     & 1000     & \textbf{0.113} & 0.659 & 3.362 \\
     & 2000       & \textbf{0.104} & 0.620 & 3.435 \\ \hline
30      & 500         & \textbf{0.183} & 0.756 & 3.455 \\
     & 1000         & \textbf{0.138} & 0.700 & 3.556  \\
     & 2000         & \textbf{0.111} & 0.635 & 3.355  \\ \hline
50      & 500         & \textbf{0.186} & 0.791 & 3.558  \\
     & 1000         & \textbf{0.170} & 0.702 & 3.632   \\
     & 2000         & \textbf{0.108} & 0.643 & 3.361  \\ 
\hhline{=========}
\end{tabular}
\caption{MSE} \label{exp_synth:MSE}
\end{subtable}

\vspace{5mm}\begin{subtable}{0.45\textwidth}  
\centering
\captionsetup{justification=centering,margin=2cm}
\begin{tabular}{cc|ccccccc}
\hhline{=========}
\textit{p} & \textit{n} & GRCI            & RCI   & RCI$_t$    & ICA & ICA$_t$ &CO    & MS           \\ \hline
10      & 500         & \cellcolor{Gray!25}1.613 & 0.003 & 0.651 & 0.182 & 0.776 & \cellcolor{Gray!25}1.208  &\cellcolor{Gray!25}1.248\\
     & 1000     & \cellcolor{Gray!25}4.075 & 0.004 & 0.832 & 0.404 & 1.073 & \cellcolor{Gray!25}3.585 &\cellcolor{Gray!25}3.686\\
     & 2000       & \cellcolor{Gray!25}13.85 & 0.009 & 1.186 & 0.914 & 1.659 & \cellcolor{Gray!25}13.43 &\cellcolor{Gray!25}13.95\\ \hline
30      & 500         & \cellcolor{Gray!25}9.199 & 0.011 & 0.720 & 0.383 & 1.075 &          \cellcolor{Gray!25}10.51 &\cellcolor{Gray!25}10.64\\
     & 1000         & \cellcolor{Gray!25}22.56 & 0.020 & 0.946 & 0.925 & 1.644 &  \cellcolor{Gray!25}24.40 & \cellcolor{Gray!25}24.80\\
     & 2000         &  \cellcolor{Gray!25}108.2 & 0.043 & 1.375 & 2.285 & 2.830 &  \cellcolor{Gray!25}111.5 & \cellcolor{Gray!25}113.6\\ \hline
50      & 500         & \cellcolor{Gray!25}32.90 & 0.033 & 0.850 & 0.650 & 1.477 &          \cellcolor{Gray!25}39.32 &\cellcolor{Gray!25}39.61\\
     & 1000         & 8\cellcolor{Gray!25}3.21 & 0.058 & 1.135 & 1.603 & 2.398 & \cellcolor{Gray!25}91.39 &  \cellcolor{Gray!25}92.28\\
     & 2000         & \cellcolor{Gray!25}222.0 & 0.125 & 1.806 & 4.145 & 4.708 &  \cellcolor{Gray!25}235.7 & \cellcolor{Gray!25}240.0\\ 
\hhline{=========}
\end{tabular}
\caption{Time in seconds} \label{exp_synth:time}
\end{subtable}
\caption{GRCI obtains the highest mean RBO values in (a) and lowest mean MSE values in (b) in every situation tested with the synthetic data. All HNM-based algorithms take approximately the same amount of time to complete as highlighted in gray in (c). } 
\end{table}

GRCI achieved the highest mean RBO in every situation (Table \ref{exp_synth:RBO}). The original version of RCI came in second place and TreeSHAP did not improve its performance. TreeSHAP improved ICA, but both versions of ICA performed much worse than GRCI and RCI. ICA frequently got stuck in local optima as evidenced by the terribly inaccurate error values when compared to RCI (Table \ref{exp_synth:MSE}). GRCI recovered the error terms about two to six times more accurately than RCI. MS and CO had the worst performances because the algorithms either recovered conditional outliers that did not induce disease or failed to output sample-specific scores. We conclude that GRCI performs the most accurately across all tested sample sizes, dimensions and metrics even after incorporating TreeSHAP into applicable alternatives. Similar results held with RBO under the PNL model, but GRCI obtained a worse MSE (Table  \ref{exp_PNL} in the Appendix).

We summarize timing results in Table \ref{exp_synth:time}. Algorithms that search over the space of HNMs -- including GRCI, CO and MS highlighted in light gray -- take about the same amount of time. These methods also expectedly take longer than the linear algorithms RCI and ICA.

\subsubsection{Real Data} We compared all seven algorithms on their ability to discover patient-specific root causes using two real datasets. Note that we do not have access to the ground truth Shapley values with real data, so we use the following modified RBO metric:
\begin{equation} \nonumber
    \frac{1}{n} \sum_{k=1}^n \sum_{i=1}^{q_k} \frac{1}{q_k} | \widehat{\mathcal{R}}_{1:i}^k \cap \mathcal{R}_{1:i}^k|/i,
\end{equation}
where we no longer weight the score by Shapley values.

\begin{figure}
\begin{subfigure}{0.45\textwidth}
    \centering
    \includegraphics[scale=0.6]{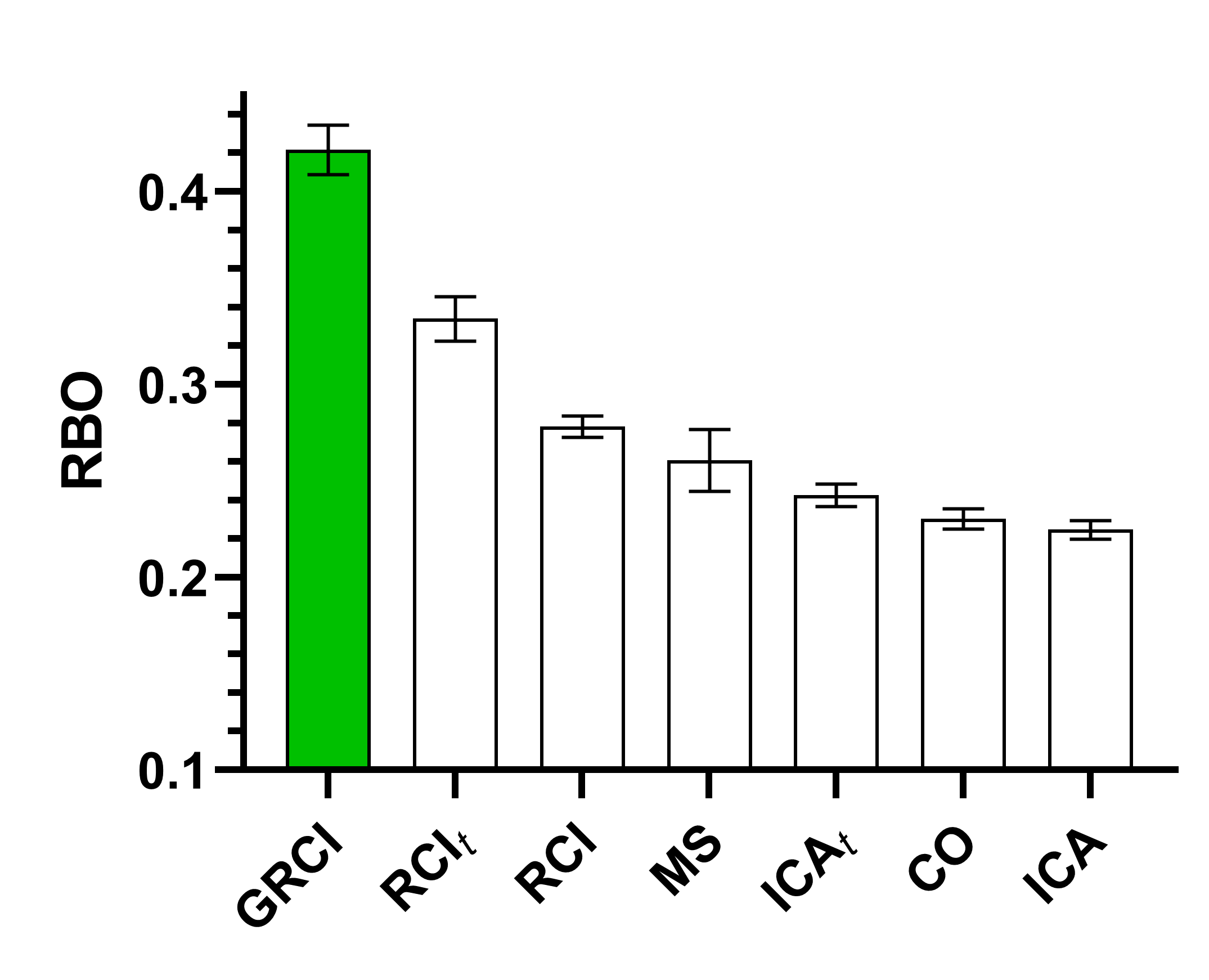}
    \caption{Primary Biliary Cirrhosis}
    \label{fig:real_PBC}
\end{subfigure}

\begin{subfigure}{0.45\textwidth}
    \centering
    \includegraphics[scale=0.6]{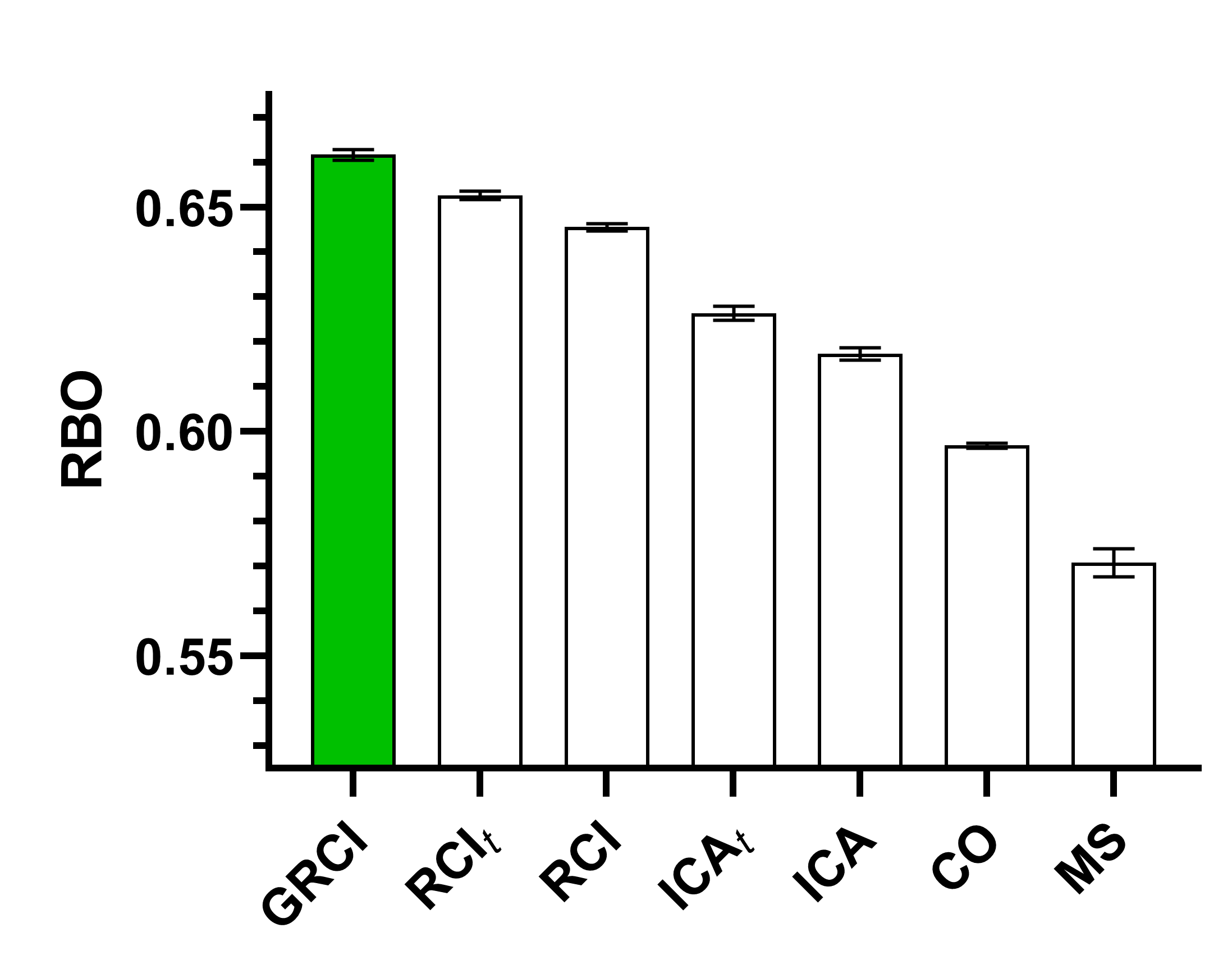}
    \caption{Pima Indians Diabetes}
    \label{fig:real_Diabetes}
\end{subfigure}
\caption{Sorted accuracy results with the real datasets.}
\end{figure}

\textbf{Primary Biliary Cholangitis. } 
The Mayo Clinic Primary Biliary Cholangitis (PBC) dataset contains samples from 258 patients with PBC who entered into a randomized clinical trial assessing the effects of medication called D-penicillamine \cite{Fleming11}. PBC is an autoimmune disease that slowly destroys the small bile ducts of the liver, eventually causing liver cirrhosis, liver decompensation and then death \cite{Hirschfield13}. The dataset contains the following continuous variables: age, bilirubin, albumin, alkaline phosphatase, copper, cholesterol, platelets, AST and pro-thrombin time.

We sought to identify the patient-specific root causes of mortality. We know that age and bilirubin cause death because older patients pass away and increased bilirubin leads to neurotoxicity \cite{Lopez14}. Intervening on the other variables does not consistently change mortality, so they are likely non-ancestors of death. High levels of bilirubin increase the frequency of death more than old age. We set the gold standard ranking as bilirubin then age if the bilirubin is at or above 2 mg/dL -- in accordance with the classic Child-Turcotte cut-off  \cite{Child64} -- and age then bilirubin otherwise.

We ran the algorithms on 1000 bootstrapped draws of the dataset. We report accuracy results among patients who passed in Figure \ref{fig:real_PBC}. GRCI achieved the best accuracy compared to all other methods. RCI came in second place in accordance with the synthetic data results. GRCI took 8 seconds on average (see the Supplementary Materials for full timing results).

\textbf{Pima Indians Diabetes. } The Pima Indians Diabetes Database is a observational dataset containing samples from females in the Pima Indian population near Pheonix, Arizona \cite{Smith88}. The dataset contains the following variables: number of pregnancies, plasma glucose concentration at two hours in an oral glucose tolerance test, diastolic blood pressure, triceps skinfold thickness, two-hour serum insulin, body mass index, diabetes pedigree function, age, and presence or absence of diabetes. 

We sought to identify the patient-specific root causes of diabetes. Recall that the incidence of diabetes increases with age, and clinicians can diagnose diabetes if the blood glucose reaches at least 200 mg/dL with a two hour oral glucose tolerance test. We therefore set the gold standard as age and glucose ranked according to their z-score in decreasing order.

We ran the algorithms again using 1000 bootstrapped draws. We reports the results for patients with diabetes in Figure \ref{fig:real_Diabetes}. GRCI again achieved the best accuracy compared to all other methods. The results with the Pima Indians Diabetes Database also mimic those seen with the PBC and synthetic datasets. GRCI took 15.2 seconds on average (Supplementary Materials). 

\begin{tcolorbox}[breakable,enhanced,frame hidden]
Summarizing the results of the patient-specific root causal inference experiments:
\begin{enumerate}[leftmargin=*,label=(\arabic*)]
    \item GRCI achieves the best performance -- in terms of both RBO and MSE -- across all sample sizes and dimensions with the synthetic data.
    \item GRCI also achieves the best performance in two real datasets with known root causes, and the real data results mimic the synthetic ones.
    \item GRCI, MS and CO take longer than the linear algorithms but still complete within about 4 minutes on average with $n=2000, p=50$.
\end{enumerate}
\end{tcolorbox}

\section{Conclusion}

We presented GRCI, the first method that generalizes the original RCI algorithm to the non-linear setting. GRCI accommodates both non-linear expectations and heteroscedastic noise under HNM. We proved identifiability of HNM in general and described a procedure that partials out both the conditional mean and MAD in a two-step regression process. We then defined patient-specific root causes using Shapley values of models predicting a diagnosis from the error terms. We introduced GRCI as an efficient method that recovers the errors by combining error extraction in functional causal models with constraint-based skeleton discovery. Experiments with both synthetic and real data highlighted considerable improvements in accurately recovering both causal direction and patient-specific root causes of disease. GRCI even outperformed other methods based on HNM engineered specifically for causal direction because GRCI does not make any Gaussian approximations. 

Experience with the GRCI algorithm however suggests several areas for improvement. First, GRCI is significantly slower than RCI both in terms of sample size and number of variables. Second, GRCI performs well under HNM but did not recover the Shapley values accurately under PNL. These results imply that the algorithm is sensitive to deviations from HNM. Finally, GRCI assumes no latent confounding, but confounders frequently exists in real data. We are not aware of a root causal contribution score that can handle confounding when investigators do not have access to the true causal graph and error term distributions. Future work could therefore improve the scalability and robustness of GRCI even in the presence of latent confounding.

\bibliographystyle{vancouver}
\bibliography{amia}  

\section*{Supplementary Materials}

\subsection*{Proofs}
\begin{reptheorem}{thm_DE}
Assume the forward model $X \rightarrow Y$ obeys HNM so that $p(x,y) = p\big(\frac{y-m(x)}{\sigma(x)}\big) p(x)$ with $m(X)$ and $\sigma(X)$ once differentiable. If there is a backward model $Y \rightarrow X$ also obeying HNM so that $p(x,y) = p\big(\frac{x-n(y)}{t(y)}\big) p(y)$, then the following differential equation holds:
\begin{equation} \nonumber
\begin{aligned}
     &-\frac{\sigma(x)}{Q(x,y)}\frac{\partial ^2}{\partial y \partial x}r(x,y)-\frac{\partial ^2}{\partial y^2}r(x,y) -\\ &\frac{\sigma^\prime(x)}{Q(x,y)}\frac{\partial }{\partial y}r(x,y)  = q^{\prime \prime}(y) + \frac{\sigma^\prime(x)}{Q(x,y)} q^\prime(y),
    \end{aligned}
\end{equation} 
where $r(x,y) = \textnormal{log } p\big(\frac{x-n(y)}{t(y)}\big), q(y) = \textnormal{log } p(y)$ both twice differentiable and $Q(x,y) = \sigma(x)m^\prime(x) + (y-m(x))\sigma^\prime(x)$. Moreover, if there exists a quadruple $(x_0,m(x_0),\sigma(x_0),p(x_0|y))$ such that $Q(x_0,y) \not = 0$ for all but countably many $y$, then $p_Y$ is completely determined by $(y_0, q^\prime(y_0))$ -- i.e., the set of all $p_Y$ satisfying the differential equation is contained in a two dimensional affine space.
\end{reptheorem}
\begin{proof}
We first derive the differential equation. Let $\pi(x,y) = \textnormal{log } p(x,y)$. The forward model allows us to write:
\begin{equation} \nonumber
    \begin{aligned}
        \frac{\partial \pi(x,y)}{ \partial y} &= \frac{ q^\prime\Big(\frac{y-m(x)}{\sigma(x)} \Big)}{\sigma(x)}\\
         \frac{\partial^2 \pi(x,y)}{ \partial y^2} &= \frac{q^{\prime\prime}\Big(\frac{y-m(x)}{\sigma(x)} \Big)}{\sigma^2(x)}\\
         \frac{\partial^2 \pi(x,y)}{ \partial y \partial x} &= \frac{-\overbrace{[\sigma(x)m^\prime(x) + (y-m(x))\sigma^\prime(x)]}^{=Q(x,y)}_{}\frac{\partial^2}{ \partial y^2} \pi(x,y)}{\sigma(x)}\\ &- \frac{\sigma^\prime(x) \frac{\partial}{ \partial y}\pi(x,y)}{\sigma(x)}.
    \end{aligned}
\end{equation}
We can also write $\frac{\partial}{ \partial y} \pi(x,y) = \frac{\partial}{ \partial y} r(x,y) + q^\prime(y)$ and likewise $\frac{\partial^2}{ \partial y^2} \pi(x,y) = \frac{\partial^2}{ \partial y^2} r(x,y) + q^{\prime\prime}(y)$ by the backward model. Observe $\frac{\partial^2 }{ \partial y \partial x} \pi(x,y)= \frac{\partial^2 }{ \partial y \partial x} r(x,y)$. Hence we have:
\begin{equation} \nonumber
\begin{aligned}
    \frac{\partial^2 r(x,y)}{ \partial y \partial x} &= \frac{-Q(x,y)\big(\frac{\partial^2}{ \partial y^2} r(x,y)+ q^{\prime\prime}(y)\big)}{\sigma(x)}\\ &- \frac{\sigma^\prime(x) \big(\frac{\partial}{ \partial y} r(x,y) + q^\prime(y) \big)}{\sigma(x)}.
\end{aligned}
\end{equation}
Rearranging the above equation leads to Equation \eqref{eq_DE}.

We now prove the second statement. Let $z(y) = q^{\prime}(y)$, $G(x,y) = -\frac{\sigma^\prime(x)}{Q(x,y)}$ and:
\begin{equation} \nonumber
\begin{aligned}
   H(x,y) = &\hspace{1mm}-\frac{\sigma(x)}{Q(x,y)}\frac{\partial ^2}{\partial x \partial y}r(x,y)-\frac{\partial ^2}{\partial y^2}r(x,y)\\ &- \frac{\sigma^\prime(x)}{Q(x,y)}\frac{\partial }{\partial y}r(x,y).     
\end{aligned}
\end{equation}
We may then write:
\begin{equation}  \nonumber
    z^\prime(y) = z(y)G(x,y) + H(x,y).
\end{equation}
Solving this linear differential equation gives:
\begin{equation}  \label{eq_expDE}
\begin{aligned}
    z(y) = c\textnormal{e}^{ \int G(x,\widetilde{y})~d\widetilde{y}} + &\textnormal{e}^{ \int G(x,\widetilde{y})~d\widetilde{y}} \times\\ 
    &\int \textnormal{e}^{ -\int G(x,\widetilde{y})~d\widetilde{y}} H(x,\widehat{y})~d\widehat{y},
\end{aligned}
\end{equation}
for some arbitrary constant $c$. Now fix $(x_0,m(x_0),\sigma(x_0),p(x_0|y))$. Then the function $z(y)$ is determined for all $y$ by \textcolor{blue}{$y_0$} and \textcolor{blue}{$z(y_0)$}, so long as $Q(x_0,y) \not = 0$ for all but countably many $y$, because we can use $y_0$ and $z(y_0)$ to find the value of $c$ (i.e., the initial condition of the solution). We can then recover $q(y)$ for all $y$ by integration (the constant follows by normalization). Thus, the set of all functions $p_Y$ satisfying Equation \eqref{eq_DE} is completely determined by \textcolor{blue}{$y_0$} and \textcolor{blue}{$q^{\prime} (y_0)$} -- a two dimensional affine space. 
\end{proof}

\begin{reptheorem}{thm:k_py}
Consider the same assumptions as Theorem \ref{thm_DE}. If both the forward and backward models follow HNM, then we have:
\begin{equation} \nonumber \label{eq_alg_MI}
\begin{aligned}
       &I(p_Y:p_{X|Y})\\ &\stackrel{+}{\geq} K(p_Y) - \inf_{(x_0,y_0)} K(x_0,m(x_0),\sigma(x_0),y_0,q^\prime(y_0)), 
\end{aligned}
\end{equation}
assuming of course that all inputs are computable.
\end{reptheorem}
\begin{proof}
Equation \eqref{eq_expDE} implies that $q^\prime(y)$ is completely determined by $(x_0,m(x_0),\sigma(x_0),y_0,q^\prime(y_0))$ given $p_{X|Y}$. We can therefore write: $K(q^\prime(y)|p_{X|Y}) \stackrel{+}{\leq} K(x_0,m(x_0),\sigma(x_0),y_0,q^\prime(y_0)|p_{X|Y})$. This holds for arbitrary $(x_0, y_0)$, so we more specifically have: $$K(q^\prime(y)|p_{X|Y}) \stackrel{+}{\leq} \inf_{(x_0,y_0)} K(x_0,m(x_0),\sigma(x_0),y_0,q^\prime(y_0)|p_{X|Y}).$$
 Note that we can recover $p_Y$ from $q^\prime$ by integration. The constant is determined by the normalization of a density. We can now write: 
\begin{equation} \nonumber
\begin{aligned}
&I(p_Y : p_{X|Y}) = K(p_Y) - K(p_Y | p_{X|Y}^*)\\ &\stackrel{+}{\geq} K(p_Y) - K(p_Y | p_{X|Y})\\
&\stackrel{+}{\geq} K(p_Y) - \inf_{(x_0,y_0)} K(x_0,m(x_0),\sigma(x_0),y_0,q^\prime(y_0)|p_{X|Y})\\
&\stackrel{+}{\geq} K(p_Y) - \inf_{(x_0,y_0)} K(x_0,m(x_0),\sigma(x_0),y_0,q^\prime(y_0)),
\end{aligned}
\end{equation}
whence the conclusion holds. 
\end{proof}

\begin{reptheorem}{thm:full}
Assume Equation \eqref{eq_HNM} is a restricted HNM according to $\mathbb{G}$. Then, $\mathbb{G}$ is uniquely identified from $\mathcal{G}$.
\end{reptheorem}
\begin{proof}
Assume that there exists another restricted HNM with graph $\underline{\mathbb{G}}$. We will show that $\mathbb{G} = \underline{\mathbb{G}}$ for any $\underline{\mathbb{G}} \in \mathcal{G}$. Assume $\mathbb{G} \not = \underline{\mathbb{G}}$. Since causal minimality holds, there must exist a directed edge $X \rightarrow Y$ in $\mathbb{G}$, and the directed edge $X \leftarrow Y$ in $\underline{\mathbb{G}}$.

Let $\bm{Q}  = \textnormal{Pa}(Y) \setminus X$ and $\bm{R}  = \textnormal{Pa}_{\underline{\mathbb{G}}}(X) \setminus Y$. Set $\bm{S} = \bm{Q} \cup \bm{R}$. Consider $\bm{S} = \bm{s}$ with $p(\bm{s})>0$. Let $X^* = (X|\bm{S} = \bm{s})$ and $Y^* = (Y|\bm{S} = \bm{s})$. Note that $\varepsilon_Y$ and $X,\bm{S}$ are d-separated in $\mathbb{G}$, so $\varepsilon_Y \ci (X,\bm{S})$ by the global Markov property. Similarly, $\varepsilon_X$ and  $Y,\bm{S}$ are d-separated in $\underline{\mathbb{G}}$, so $\varepsilon_X \ci (Y,\bm{S})$. Peters and colleagues showed that $g(X^*,\bm{q}, \varepsilon_Y) \stackrel{d}{=} g(X,\bm{Q},\varepsilon_Y)|\bm{s}$ for any measurable function $g$ in their Lemma 36 so long as $p(\bm{s})>0$ (and likewise for $Y^*$) \cite{Peters14}. Applying this result gives the bivariate model:
\begin{equation} \nonumber
\begin{aligned}
    Y^* &= m_Y(\bm{q},X^*) + \varepsilon_Y \sigma_Y(\bm{q},X^*) \textnormal{ with } \varepsilon_Y \ci X^* \textnormal{ in } \mathbb{G},\\
    X^* &= m_X(\bm{r},Y^*) + \varepsilon_X \sigma_X(\bm{r},Y^*) \textnormal{ with } \varepsilon_X \ci Y^* \textnormal{ in } \underline{\mathbb{G}}.
\end{aligned}
\end{equation}
But we chose $\bm{s}$ such that $p(x,y|\bm{s})$ does not satisfy Equation \eqref{eq_DE} -- a contradiction of Theorem \ref{thm_DE}. 
\end{proof}

\newpage
\subsection*{Additional Results}

\begin{figure}[h]
\begin{subfigure}{0.45\textwidth}
    \centering
    \includegraphics[scale=0.6]{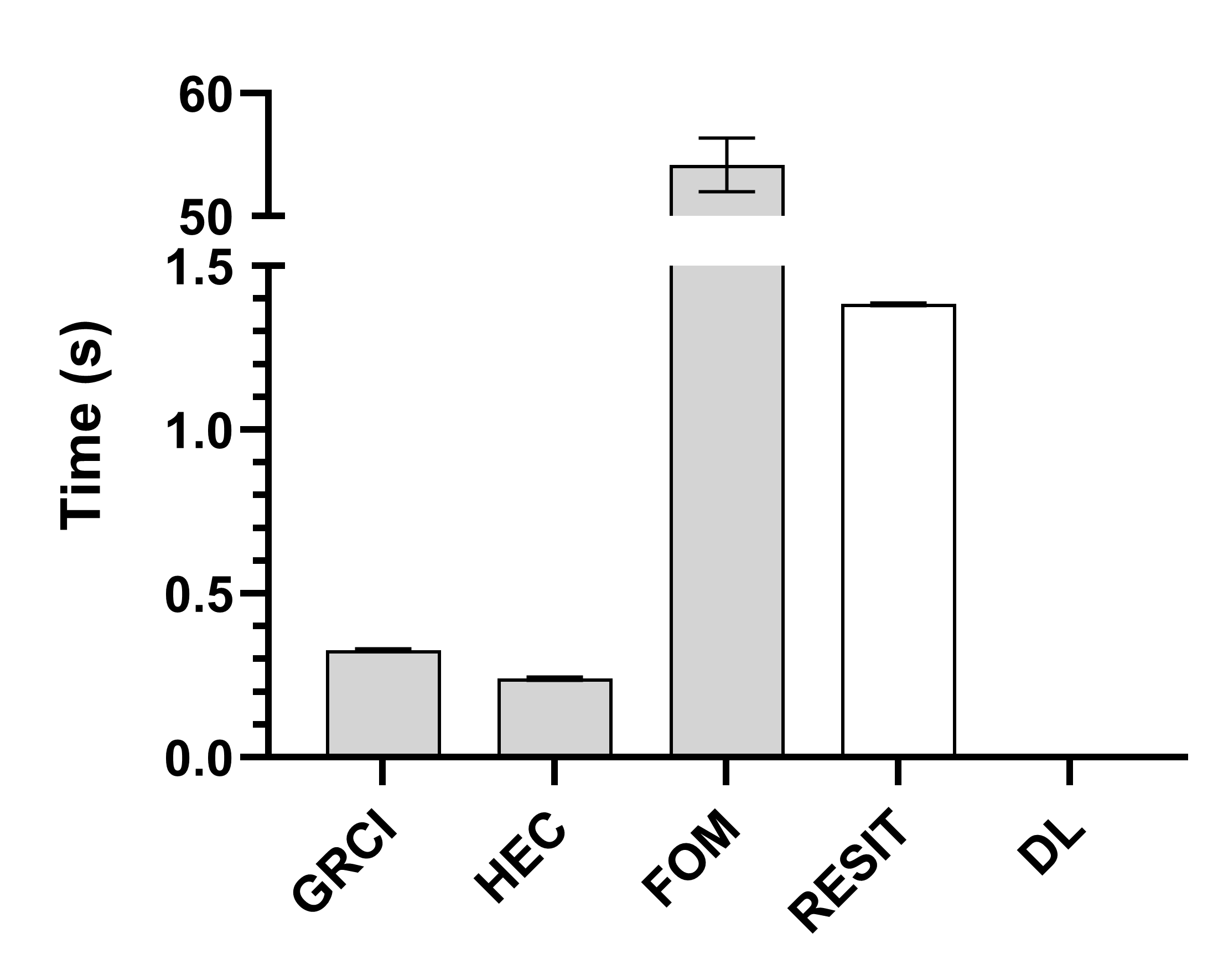}
    \caption{Synthetic pairs}
    \label{fig:pairs_synth_time}
\end{subfigure}

\begin{subfigure}{0.45\textwidth}
    \centering
    \includegraphics[scale=0.6]{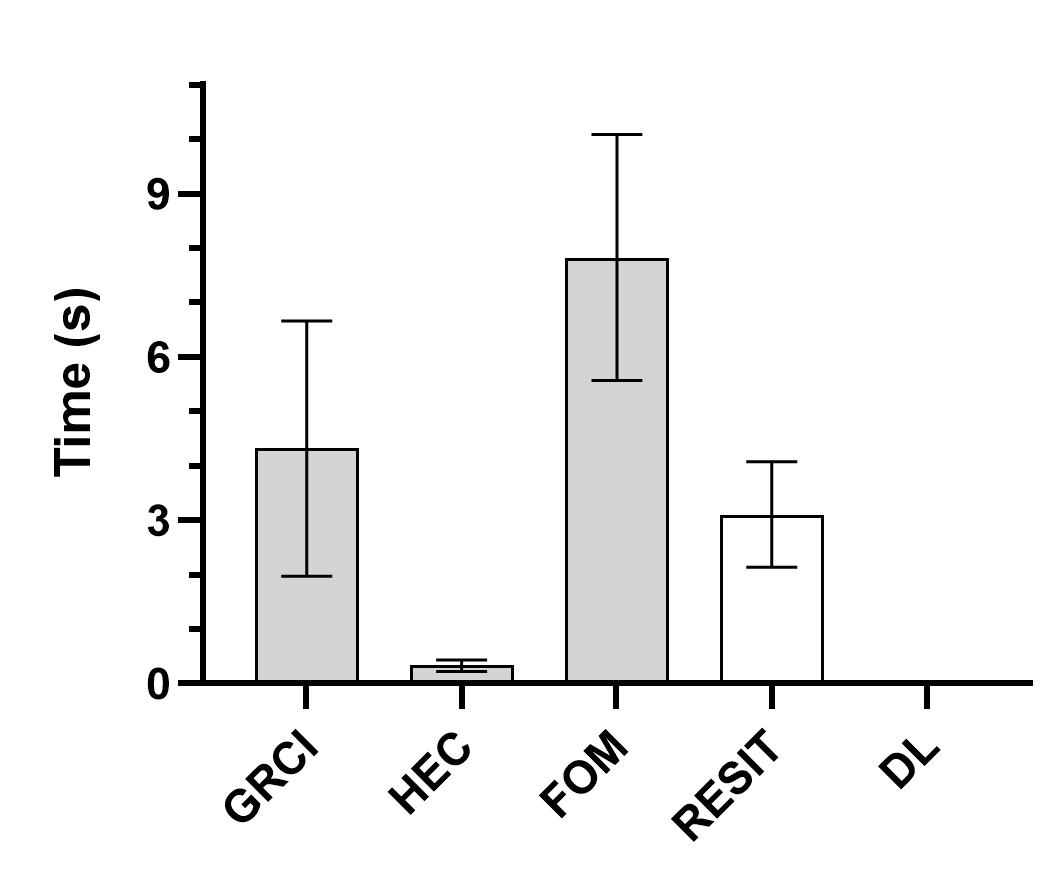}
    \caption{T\"{u}bingen cause-effect pairs}
    \label{fig:pairs_Tuben}
\end{subfigure}
\caption{Timing results for the causal direction experiments. HNM-based methods are highlighted in gray. GRCI takes longer than RESIT in (b) because we capped the RESIT sample size at 3000 due to its scaling issues.}
\end{figure}

\begin{table}[]
\begin{subtable}{0.45\textwidth}  
\begin{tabular}{cc|ccccccc}
\hhline{=========}
\textit{p}          & \textit{n} & GRCI           & RCI   & RCIt  & ICA   & ICAt  & CO    & MS    \\ \hline
10 & 500        & \textbf{0.740} & 0.671 & 0.677 & 0.632 & 0.677 & 0.598 & 0.490 \\
                    & 1000       & \textbf{0.760} & 0.685 & 0.689 & 0.657 & 0.689 & 0.601 & 0.486 \\
                    & 2000       & \textbf{0.783} & 0.700 & 0.721 & 0.667 & 0.721 & 0.590 & 0.458 \\ \hline
30 & 500        & \textbf{0.678} & 0.630 & 0.641 & 0.542 & 0.641 & 0.482 & 0.385 \\
                    & 1000       & \textbf{0.721} & 0.667 & 0.681 & 0.624 & 0.681 & 0.466 & 0.397 \\
                    & 2000       & \textbf{0.746} & 0.689 & 0.709 & 0.653 & 0.709 & 0.480 & 0.391 \\ \hline
50 & 500        & \textbf{0.664} & 0.595 & 0.614 & 0.293 & 0.614 & 0.394 & 0.332 \\
                    & 1000       & \textbf{0.707} & 0.630 & 0.650 & 0.542 & 0.650 & 0.390 & 0.329 \\
                    & 2000       & \textbf{0.741} & 0.657 & 0.678 & 0.635 & 0.678 & 0.378 & 0.296\\
                    \hhline{=========}
\end{tabular}

\caption{RBO} \label{exp_PNL:RBO}
\end{subtable}

\vspace{5mm}\begin{subtable}{0.45\textwidth}  
\centering
\captionsetup{justification=centering,margin=2cm}
\begin{tabular}{cc|ccccccc}
\hhline{=========}
\textit{p} & \textit{n} & GRCI            & RCI   & ICA            \\ \hline
10      & 500         & \textbf{0.219} & 0.756 & 3.262  \\
     & 1000     & \textbf{0.223} & 0.714 & 2.947 \\
     & 2000       & \textbf{0.213} & 0.671 & 3.431 \\ \hline
30      & 500         & \textbf{0.268} & 0.717 & 3.302 \\
     & 1000         & \textbf{0.249} & 0.675 & 3.365  \\
     & 2000         & \textbf{0.251} & 0.643 & 3.329  \\ \hline
50      & 500         & \textbf{0.277} & 0.779 & 3.536  \\
     & 1000         & \textbf{0.240} & 0.739 & 3.473   \\
     & 2000         & \textbf{0.228} & 0.685 & 3.189  \\ 
\hhline{=========}
\end{tabular}
\caption{MSE} \label{exp_PNL:MSE}
\end{subtable}

\vspace{5mm}\begin{subtable}{0.45\textwidth}  
\centering
\captionsetup{justification=centering,margin=2cm}
\begin{tabular}{cc|ccccccc}
\hhline{=========}
\textit{p} & \textit{n} & GRCI            & RCI   & RCI$_t$    & ICA & ICA$_t$ &CO    & MS           \\ \hline
10      & 500         & \cellcolor{Gray!25}1.851 & 0.003 & 0.584 & 0.262 & 0.762 & \cellcolor{Gray!25}1.598  &\cellcolor{Gray!25}1.562\\
     & 1000     & \cellcolor{Gray!25}4.943 & 0.008 & 0.860 & 0.598 & 1.202 & \cellcolor{Gray!25}4.733 &\cellcolor{Gray!25}4.614\\
     & 2000       & \cellcolor{Gray!25}16.97 & 0.014 & 1.398 & 1.359 & 2.057 & \cellcolor{Gray!25}17.74 &\cellcolor{Gray!25}17.05\\ \hline
30      & 500         & \cellcolor{Gray!25}12.84 & 0.018 & 0.709 & 0.618 & 1.312 &          \cellcolor{Gray!25}15.17 &\cellcolor{Gray!25}15.00\\
     & 1000         & \cellcolor{Gray!25}31.75 & 0.031 & 1.066 & 1.463 & 2.230 &  \cellcolor{Gray!25}35.03 & \cellcolor{Gray!25}34.48\\
     & 2000         &  \cellcolor{Gray!25}223.9 & 0.060 & 1.766 & 3.533 & 4.129 &  \cellcolor{Gray!25}230.8 & \cellcolor{Gray!25}227.9\\ \hline
50      & 500         & \cellcolor{Gray!25}36.77 & 0.044 & 0.788 & 0.650 & 1.725 &          \cellcolor{Gray!25}44.24 &\cellcolor{Gray!25}43.89\\
     & 1000         & \cellcolor{Gray!25}85.59 & 0.077 & 1.245 & 1.603 & 3.302 & \cellcolor{Gray!25}95.99 &  \cellcolor{Gray!25}94.93\\
     & 2000         & \cellcolor{Gray!25}255.2 & 0.161 & 2.233 & 4.145 & 6.234 &  \cellcolor{Gray!25}273.63 & \cellcolor{Gray!25}268.7\\ 
\hhline{=========}
\end{tabular}
\caption{Time in seconds} \label{exp_PNL:time}
\end{subtable}
\caption{Accuracy and timing results with the PNL model. The algorithms achieved average RBO values comparable to HNM, but GRCI obtained substantially worse average MSE.} \label{exp_PNL}
\end{table}

\begin{figure}[h]
\begin{subfigure}{0.45\textwidth}
    \centering
    \includegraphics[scale=0.6]{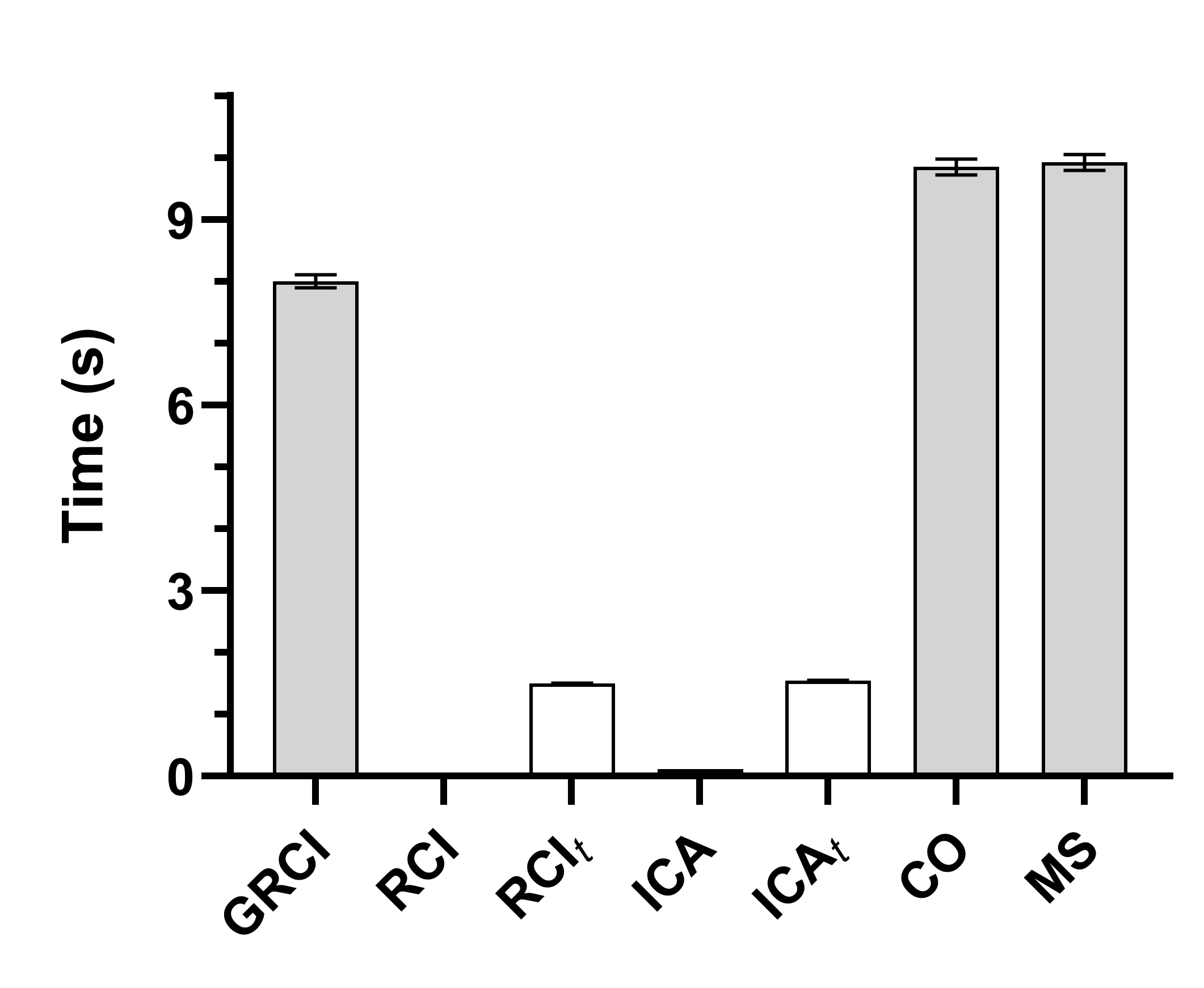}
    \caption{Primary Biliary Cirrhosis}
    \label{fig:real_PBC_time}
\end{subfigure}

\begin{subfigure}{0.45\textwidth}
    \centering
    \includegraphics[scale=0.6]{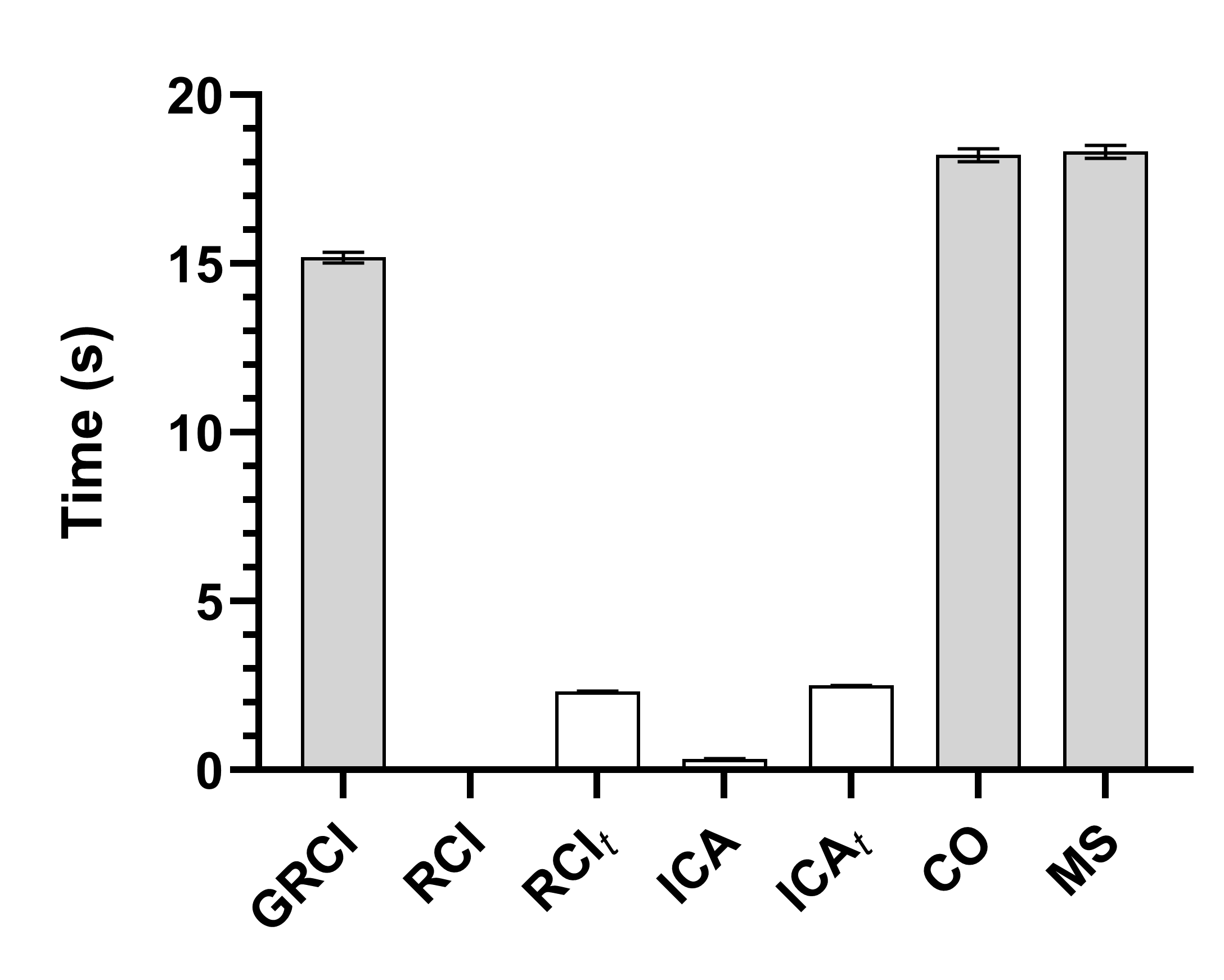}
    \caption{Pima Indians Diabetes}
    \label{fig:real_Diabetes_time}
\end{subfigure}
\caption{Timing results for the real datasets.}
\end{figure}

\end{document}